\documentclass[10pt, a4paper]{article}
\usepackage{amsthm, amssymb, amsmath}
\usepackage{verbatim, float}
\usepackage{mathrsfs}
\usepackage{graphics}
\usepackage[usenames,dvipsnames]{pstricks}
\usepackage{epsfig}
\usepackage{pst-grad} % For gradients
\usepackage{pst-plot} % For axes
\usepackage{cases}
\usepackage{algorithm,algorithmic}

\usepackage[top=2.4cm, bottom=2.4cm, left=2.5cm, right=2.5cm]{geometry}
\usepackage{url}

\newcommand{\nc}{\newcommand}

\newcommand{\Sn}{{\mathbb{S}_n}}

\theoremstyle{definition}
 
\newtheorem{theorem}{Theorem}[section]

\newtheorem{lemma}[theorem]{Lemma}

\theoremstyle{remark}

\newtheorem{remark}{Remark}[section]
\nc\remove[1]{}

\nc\bfa{{\boldsymbol a}}\nc\bfA{{\mathbf A}}\nc\cA{{\mathcal A}}
\nc\bfb{{\boldsymbol b}}\nc\bfB{{\mathbf B}}\nc\cB{{\mathcal B}}
\nc\bfc{{\boldsymbol c}}\nc\bfC{{\mathbf C}}\nc\cC{{\mathcal C}}
\nc\bfd{{\boldsymbol d}}\nc\bfD{{\mathbf D}}\nc\cD{{\mathcal D}}\nc\sD{{\mathscr D}}
\nc\bfe{{\boldsymbol e}}\nc\bfE{{\mathbf E}}\nc\cE{{\EuScript E}}
\nc\bff{{\boldsymbol f}}\nc\bfF{{\mathbf F}}\nc\cF{{\mathcal F}}
\nc\bfg{{\boldsymbol g}}\nc\bfG{{\mathbf G}}\nc\cG{{\mathcal G}}
\nc\bfh{{\boldsymbol h}}\nc\bfH{{\mathbf H}}\nc\cH{{\mathcal H}}
\nc\bfi{{\boldsymbol i}}\nc\bfI{{\mathbf I}}\nc\cI{{\mathcal I}}
\nc\bfj{{\boldsymbol j}}\nc\bfJ{{\mathbf J}}\nc\cJ{{\mathcal J}}
\nc\bfk{{\boldsymbol k}}\nc\bfK{{\mathbf K}}\nc\cK{{\mathcal K}}
\nc\bfl{{\boldsymbol l}}\nc\bfL{{\mathbf L}}\nc\cL{{\mathcal L}}\nc\sL{{\mathscr L}}
\nc\bfm{{\boldsymbol m}}\nc\bfM{{\mathbf M}}\nc\cM{{\mathcal M}}
\nc\bfn{{\boldsymbol n}}\nc\bfN{{\mathbf N}}\nc\cN{{\mathcal N}}
\nc\bfo{{\boldsymbol o}}\nc\bfO{{\mathbf O}}\nc\cO{{\mathcal O}}
\nc\bfp{{\boldsymbol p}}\nc\bfP{{\mathbf P}}\nc\cP{{\mathcal P}}
\nc\bfq{{\boldsymbol q}}\nc\bfQ{{\mathbf Q}}\nc\cQ{{\mathcal Q}}\nc\sQ{{\mathscr Q}}
\nc\bfr{{\boldsymbol r}}\nc\bfR{{\mathbf R}}\nc\cR{{\mathcal R}}
\nc\bfs{{\boldsymbol s}}\nc\bfS{{\mathbf S}}\nc\cS{{\mathcal S}}
\nc\bft{{\boldsymbol t}}\nc\bfT{{\mathbf T}}\nc\cT{{\mathcal T}}\nc\sT{{\mathscr T}}
\nc\bfu{{\boldsymbol u}}\nc\bfU{{\mathbf U}}\nc\cU{{\mathcal U}}
\nc\bfv{{\boldsymbol v}}\nc\bfV{{\mathbf V}}\nc\cV{{\mathcal V}}
\nc\bfw{{\boldsymbol w}}\nc\bfW{{\mathbf W}}\nc\cW{{\mathcal W}}\nc\sW{{\mathscr W}}
\nc\bfx{{\boldsymbol x}}\nc\bfX{{\mathbf Z}}\nc\cX{{\EuScript X}}
\nc\bfy{{\boldsymbol y}}\nc\bfY{{\mathbf Y}}\nc\cY{{\EuScript Y}}\nc\sY{{\mathscr Y}}
\nc\bfz{{\boldsymbol z}}\nc\bfZ{{\mathbf Z}}\nc\cZ{{\mathcal Z}}\nc\sZ{{\mathscr Z}}

\def\h_q{\qopname\relax{no}{h_q}}
\def\h{\qopname\relax{no}{h}}

\def\prob{{\mathbb P}}

\begin{document}
%
%\twocolumn[
%
%\aistatstitle{Efficient Rank Aggregation via Lehmer Codes}
%
%\aistatsauthor{ Anonymous Author 1 \And Anonymous Author 2 \And Anonymous Author 3 }
%
%\aistatsaddress{ Unknown Institution 1 \And Unknown Institution 2 \And Unknown Institution 3 } ]

\title{ Efficient Rank Aggregation via Lehmer Codes}
\renewcommand*{\thefootnote}{\fnsymbol{footnote}}
%\author[1]{Pan Li\thanks{panli2@illinois.edu}}
%\author[2]{Arya Mazumdar\thanks{arya@cs.umass.edu}}
%\author[1]{Olgica Milenkovic\thanks{milenkov@illinois.edu}}
%\affil[1]{Department of ECE, University of Illinois Urbana-Champaign}
%\affil[2]{Department of CS,  University of Massachusetts}
\title{Efficient Rank Aggregation via Lehmer Codes\footnotemark[1]}
\author{ Pan~Li~\footnotemark[2], Arya~Mazumdar~\footnotemark[3] and Olgica~Milenkovic~\footnotemark[2] }
\footnotetext[1]{A shorter version of this will appear in Artificial Intelligence and Statistics (AISTATS), 2017. } 
   \footnotetext[2]{The authors are with the Coordinated Science Laboratory, Department of Electrical and Computer Engineering, University of Illinois at Urbana-Champaign (email: panli2@illinois.edu, milenkov@illinois.edu)}
   \footnotetext[3]{The author is with Department of Computer Science, University of Massachusetts(email: arya@cs.umass.edu).
   }
\date{}
\maketitle
\renewcommand*{\thefootnote}{\arabic{footnote}}

\begin{abstract} We propose a novel rank aggregation method based on converting permutations into their 
corresponding Lehmer codes or other subdiagonal images. Lehmer codes, also known as inversion vectors, 
are vector representations of permutations in which each coordinate can take values not restricted by the values of other coordinates. This transformation allows for decoupling of the coordinates and for performing 
aggregation via simple scalar median or mode computations. We present simulation results illustrating the performance of this completely parallelizable approach and analytically prove that both the mode and median aggregation procedure recover the correct centroid aggregate with small sample complexity when the permutations are drawn according to the well-known Mallows models. The proposed Lehmer code approach may also be used on partial rankings, with similar performance guarantees. 
\end{abstract}

\section{Introduction}

Rank aggregation is a family of problems concerned with fusing disparate ranking information, and it 
arises in application areas as diverse as social choice, meta-search, natural language processing, bioinformatics, and information retrieval~\cite{burges2005learning,liu2009learning,kim2014hydra}. The observed rankings are either linear orders (permutations) or partial (element-tied) rankings\footnote{In the mathematics literature, partial rankings are commonly referred to as weak orders, while the term partial order is used to describe orders of subsets of elements of a ground set. We nevertheless use the term partial ranking to denote orders with ties, as this terminology is more widely adopted by the machine learning community.}. Sometimes, rankings are assumed to be of the form of a set of pairwise comparisons~\cite{negahban2012iterative,chen2013pairwise}. Note that, many massive ordinal datasets arise from \emph{ratings}, rather than actual comparisons. 
%, we focus our attention on aggregating permutations and partial orders containing more than two elements. 
Rank aggregation, rather than averaging of ratings, is justified due to the fact that most raters have different rating ``scales''.
As an example, the rating three of one user may indicate that the user liked the item, while the rating three by another user may indicate that the user  disliked the item. Hence, actual preferences can only be deduced using ranked ratings. 

In rank aggregation, the task at hand is to find a ranking that is at the smallest cumulative distance from a given set of rankings. Here, the cumulative distance from a set equals the sum of the distances from each element of the set, and the most frequently used distance measure for the case of permutations is the Kendall $\tau$ distance. For the case of partial rankings, the distance of choice is the Kemeny distance~\cite{kemeny1959mathematics}. The Kendall $\tau$ distance between two permutations equals the smallest number of adjacent transpositions needed to convert one permutation into the other. The Kemeny distance contains an additional weighted correction term that accounts for ties in the rankings.

It is well known that for a wide range of distance functions, learning the underlying models and aggregating rankings is computationally hard~\cite{davenport2004computational}. Nevertheless, for the case when the distance measure is the Kendall $\tau$ distance, a number of approximation algorithms have been developed that offer various trade-offs between quality of aggregation and computational complexity~\cite{dwork2001rank,ailon2008aggregating}. The techniques used for aggregating permutations in a given set include randomly choosing a permutation from the set (PICK-A-PERM), pivoting via random selections of elements and divide-and-conquer approaches (FAS-PIVOT), Markov chain methods akin to PageRank, and minimum weight graph matching methods exploiting the fact that the Kendall $\tau$ distance is well-approximated by the Spearman footrule distance (SM)~\cite{diaconis1977spearman}. Methods with provable performance guarantees -- PICK-A-PERM, FAS-PIVOT, and SM -- give a $2$-approximation for the objective function, although combinations thereof are known to improve the constant to $11/7$ or  $4/3$~\cite{ailon2008aggregating}. There also exists a polynomial time approximation scheme (PTAS) for the aggregation problem~\cite{kenyon2007rank}.

Unfortunately, most of these known approximate rank aggregation algorithms  have high complexity for use with massive datasets 
and may not be implemented in a parallel fashion. Furthermore, they do not easily extend to partial rankings. In many cases, a performance analysis on probabilistic models~\cite{fligner1993probability} such as the Plackett-Luce model~\cite{caron2012efficient} or the Mallows model~\cite{lu2011learning,lebanon2002cranking}, is intractable. 

In this paper, we propose a new approach to the problem of rank aggregation that uses a combinatorial transform, the Lehmer code (LC). The gist of the approach is to convert permutations into their Lehmer code representations, in which each coordinate takes values independently from other coordinates. Aggregation over the Lehmer code domain reduces to computing the median or mode of a bounded set of numbers, which can be done in linear time. Furthermore, efficient conversion algorithms between permutations and Lehmer codes -- also running in linear time -- are known, making the overall complexity of the parallel implementation of the scheme $O(m+n)$, where $m$ denotes the number of permutations to be aggregated, and $n$ denotes the length (size) of the permutations. To illustrate the performance of the Lehmer code aggregators (LCAs) on permutations, we carry out simulation studies showing that the algorithms perform comparably with the best known methods for approximate aggregation, but at a significantly lower computational cost. We then proceed to establish a number of theoretical performance guarantees for the LCA algorithms: In particular, we consider the Mallows model with the Kendall $\tau$ distance for permutations and Kemeny distance for partial rankings where ties are allowed. We show that the centroid permutation of the model or a derivative thereof may be recovered from $O(\log\,n)$ samples from the corresponding distribution with high probability. 
%The  conclude with a short discussion of other subdiagonal rank representations that may be used for aggregation purposes.
 
The paper is organized as follows. Section~\ref{sec:prelim} contains the mathematical preliminaries and the definitions used throughout the paper. Section~\ref{sec:algorithms} introduces our new aggregation methods for two types of rankings, while Section~\ref{sec:analysis} describes our analysis pertaining to the Mallows and generalized Mallows models. Section~\ref{sec:simulations} contains illustrative simulation results comparing the performance of the LC aggregators to that of other known aggregation methods, both on simulated and real ranking data. A number of technical results, namely detailed proofs of theorems and lemmas, can be found in the Appendix.

\section{Mathematical Preliminaries} \label{sec:prelim}

Let $S$ denote a set of $n$ elements, which without loss of generality we assume to be equal to $[n]\equiv\{{1,2,\ldots,n\}}$. 
A ranking is an ordering of a subset of elements $Q$ of $[n]$ according to a 
predefined rule. When $Q=[n]$, we refer to the order as a permutation (full ranking). When a ranking includes ties, we refer to it as a partial ranking (weak or bucket order). Partial rankings may be used to complete rankings of subsets of element in $[n]$ in a number of different ways~\cite{fagin2004comparing}, one being to tie all unranked elements at the last position.
%Later, we will also call $\sigma_{A}$ as A-map $\sigma$ for somewhere necessary.

Rigorously, a permutation is a bijection $\sigma \, : \, [n] \rightarrow [n]$, and the set of permutations over $[n]$ forms the symmetric group of order $n!$ denoted by $\Sn$. 
For any $\sigma \in \Sn$ and $x \in [n]$, $\sigma(x)$ denotes the rank (position) of the element $x$ in $\sigma$. %\textcolor{blue}{Please note that I defined the ``higher'' and ``lower'' in the other way around. I will try to revise the following parts including supplement to be consistent with yours. When you do proof-reading, please be careful about ``high'' and ``low''.}  
We say that $x$ is ranked higher than $y$ (ranked lower than $y$) iff $\sigma(x)<\sigma(y)$ ($\sigma(x)>\sigma(y)$). The inverse of a permutation $\sigma$ is denoted by $\sigma^{-1}: [n]\rightarrow [n]$. Clearly, $\sigma^{-1}(i)$ represents the element ranked at position $i$ in $\sigma$. 
%It is worth pointing out that with these definitions, we diverged from the more common approach to define a permutation $\sigma$ by assuming that $\sigma(x)$ represent the \emph{element} ranked at position $x$ in $\sigma$. The reason behind this change is to avoid frequent use of the inverse function, as most of our proofs involve positional information. 
We define \emph{the projection of a permutation} $\sigma$ over a subset of elements $Q\subseteq [n]$, denoted by $\sigma_{Q}: Q\rightarrow [|Q|]$, as an ordering of elements in $Q$ such that $x,y\in Q$, $\sigma_{Q}(x)>\sigma_{Q}(y)$ iff $\sigma(x)>\sigma(y)$. As an example, the projection of $\sigma=(2,1,4,5,3,6)$ over $Q=\{{1,3,5,6\}}$ equals $\sigma_Q=(1,3,2,4),$ since $\sigma(1)<\sigma(5)<\sigma(3)<\sigma(6)$.  As can be seen, $\sigma_{Q}(x)$ equals the rank of element $x \in Q$ in $\sigma$.

We use a similar set of definitions for partial rankings~\cite{fagin2004comparing}. A partial ranking $\sigma$ is also defined as a mapping $[n]\rightarrow [n]$. In contrast to permutations, where the mapping is a bijection, the mapping in partial ranking allows for ties, i.e., there may exist two elements $x \neq y$ such that $\sigma(x)=\sigma(y)$. A partial ranking is often represented using buckets, and is in this context referred to as a \emph{bucket order}~\cite{fagin2004comparing}. In a bucket order, the elements of the set $[n]$ are partitioned into a number of subsets, or buckets, $\mathcal{B}_1,\mathcal{B}_2,...,\mathcal{B}_t$. We let $\sigma(x)$ denote the index of the bucket containing the element $x$ in $\sigma$, so the element $x$ is assigned to bucket $\mathcal{B}_{\sigma(x)}$. Two elements $x,y$ lie in the same bucket if and only if they are tied in $\sigma$. We may also define a projection of a partial ranking $\sigma$ over a subset of elements $Q\subset [n]$, denoted by $\sigma_{Q}$, so that for $x,y\in Q$, $\sigma_{Q}(x)>\sigma_{Q}(y)$ iff $\sigma(x)>\sigma(y)$ and $\sigma_{Q}(x)=\sigma_{Q}(y)$ iff $\sigma(x)=\sigma(y)$. %In the following, we use the buckets of $\sigma$ to indicate the above several bucket that hold different segments of ties in $\sigma$.
%%%%%%%%%%%%%%%%%%%%%%%%%%%%%%%%%%%%%%%%%%%%
For a given partial ranking $\sigma$, we use $\mathcal{B}_1(\sigma),\mathcal{B}_2(\sigma),...,\mathcal{B}_t(\sigma)$ to denote its corresponding buckets. In addition, we define $r_{k(\sigma)}\triangleq\sum_{j=1}^{k}|\mathcal{B}_j(\sigma)|$ and $l_{k(\sigma)}\triangleq\sum_{j=1}^{k-1}|\mathcal{B}_j(\sigma)|+1$. Based on the previous discussion, $r_{\sigma(x)}(\sigma)-l_{\sigma(x)}(\sigma)+1=|\mathcal{B}_{\sigma(x)}(\sigma)|$ (the number of elements that are in the bucket containing $x$). When referring to the bucket for a certain element $x$,  we use $\mathcal{B}_{\sigma(x)},\,r_{\sigma(x)},\,l_{\sigma(x)}$ whenever no confusion arises. Note that if we arbitrarily break ties in $\sigma$ to create a permutation $\sigma'$, then $l_{\sigma(x)}\leq \sigma'(x)\leq r_{\sigma(x)}$; clearly, if $\sigma$ is a permutation, we have $l_{\sigma(i)}=\sigma(i)=r_{\sigma(i)}$. 

%All the elements in $\sigma$ are distributed into different buckets. If $x\in \mathcal{B}_g(\sigma),\; y\in \mathcal{B}_h(\sigma)$, then $i<_{\sigma} j$ iff $g<h$ while $i=_{\sigma} j$ iff $g=h$. 
% To be convenient, we $\sigma(i)$ denote the index of bucket that element $i$ belongs to, i.e., $i\in \mathcal{B}_{\sigma_{i}}(\sigma)$. 
%%%%%%%%%%%%%%%%%%%%%%%%%%%%%%%%%%%%

A number of distance functions between permutations are known from the social choice, learning and discrete mathematics literature~\cite{diaconis1977spearman}. 
One distance function of interest is based on transpositions: A transposition $(a,b)$ is a swap of elements at positions $a$ and $b$, $a \neq b$. If $|a-b|=1$, the transposition is referred to as an adjacent transposition. It is well known that transpositions (adjacent transpositions) generate $\mathbb{S}_n$, i.e., any permutation $\pi\in\mathbb{S}_n$ can be converted into another permutation $\sigma\in\mathbb{S}_n$ through a sequence of transpositions (adjacent transpositions)~\cite{stanley2011enumerative}. The smallest number of adjacent transpositions needed to convert a permutation $\pi$ into another permutation $\sigma$ is known as the Kendall
$\tau$ distance between $\pi$ and $\sigma$, and is denoted by $d_{\tau}(\pi,\sigma)$. Alternatively, the Kendall $\tau$ distance between two permutations $\pi$ and $\sigma$ over $[n]$ equals the number of mutual inversions between the elements of the two permutations:
\begin{align}\label{fullmetric}
d_{\tau}(\sigma,\pi)= |\{(x,y):  \pi(x) > \pi(y),\sigma(x) < \sigma(y)\}|.
\end{align}
Another distance measure, that does not rely on transpositions, is the Spearman footrule, defined as 
$$d_S(\sigma,\pi) = \sum_{x\in [n]} |\sigma(x) -\pi(x)|.$$ 
A well known result by Diaconis and Graham~\cite{diaconis1977spearman} 
asserts that $d_{\tau}(\pi,\sigma) \leq d_S(\pi,\sigma) \leq 2 d_{\tau}(\pi,\sigma)$.

One may also define an extension of the Kendall $\tau$ distance for the case of two partial rankings $\pi$ and $\sigma$ over the set $[n]$, known as the Kemeny distance: 
\begin{align}\label{partialmetric}
d_{K}(\pi,\sigma)=&|\{(x,y): \pi(x)>\sigma(y),\pi(x)<\sigma(y)\}|\nonumber\\
+&\frac{1}{2}|\{(x,y): \pi(x)=\pi(y),\sigma(x)>\sigma(y), \nonumber \\
\text{or}&\;\pi(x)>\pi(y),\sigma(x)=\sigma(y), \}|.
\end{align}
The Kemeny distance includes a component equal to the Kendal $\tau$ distance between the linear chains in the partial rankings, and another, scaled component that characterizes the distance of tied pairs of elements~\cite{fagin2004comparing}. The Spearman footrule distance may also be defined to apply to partial rankings~\cite{fagin2004comparing}, and it equals the sum of the absolute differences between ``positions'' of elements in the partial rankings. Here, the position of an element $x$ in a partial ranking $\sigma$ is defined as
$$\text{pos}_{\sigma}(x)\triangleq\sum_{j=1}^{\sigma(x)-1}|\mathcal{B}_{j}(\sigma)|+\frac{|\mathcal{B}_{\sigma(x)}(\sigma)|+1}{2}.$$ 
The above defined Spearman distance is a $2$-approximation for the Kemeny distance between two partial rankings~\cite{fagin2004comparing}.

A permutation $\sigma=(\sigma(1),\ldots,\sigma(n)) \in \Sn$ may be uniquely represented via its \emph{Lehmer code} (also called the {\em inversion vector}), i.e. 
a word of the form
$$ \bfc_\sigma\in \mathcal{C}_n\triangleq \{0\} \times [0,1]\times [0,2] \times\dots\times [0,n-1],$$ 
where for $i=1,\dots,n$,
\begin{align} \label{DefLC}
\bfc_\sigma(x)= |\{y: \, y<x, \sigma(y)>\sigma(x)\}|,
\end{align}
and for integers $a\le b$, $[a,b] \equiv [a, a+1, \dots, b]$.
By default, $\bfc_\sigma(1)=0$, and is typically omitted.
%$\bfx_\sigma(i)=\sum_{j=1}^{i-1} I\{{\sigma(f^{-1}(i))<\sigma(f^{-1}(j))\}},$ where $I\{{E\}}$ denotes the indicator function of the event $E$. In words, $\bfx_\sigma(i), i=1,\dots,n-1$ is the number of inversions (pairs out of order) in the permutation $\sigma$ for which $i+1$ is the first element.
For instance, we have
%supposing $\sigma=(x_2,x_1,x_6,x_4,x_3,x_7,x_5,x_9,x_8)$ and $f(x_i)=i$, we have
%\begin{center}
%\vspace*{.05in}
%       \begin{tabular}{c@{\hspace*{.6in}}c}
%  	       $\sigma$                 & $\bfc_\sigma$\\
%     $x_2$ $x_1$ $x_6$ $x_4$ $x_3$ $x_7$ $x_5$ $x_9$ $x_8$ &1 0 1 0 3 1 0 1
%      \end{tabular}
%\vspace*{.05in}
%\end{center}
\begin{center}
\vspace*{.05in}
       \begin{tabular}{c@{\hspace*{.2in}}ccccccccc}
       		%$S$ & $x_1$&$x_2$& $x_3$&$x_4$&$x_5$&$x_6$&$x_7$&$x_8$&$x_9$ \\
		$e$ & 1&2&3&4&5&6&7&8&9\\
  	       $\sigma$  & 2&1&4&5&7&3&6&9&8 \\
	       $\bfc_{\sigma}$  & 0&1&0&0&0&3&1&0&1 \\
	       %$\bfc_\sigma$ 0\\
     %$x_2$ $x_1$ $x_6$ $x_4$ $x_3$ $x_7$ $x_5$ $x_9$ $x_8$ &1 0 1 0 3 1 0 1
      \end{tabular}
\vspace*{.05in}
\end{center}
It is well known that the Lehmer code is bijective, and that the encoding and decoding algorithms have linear complexity $(n)$~\cite{marevs2007linear,myrvold2001ranking}. Codes with similar properties to the Lehmer codes have been extensively studied under the name of \emph{subdiagonal codes}. An overview of such codes and their relationship to Mahonian statistics on permutations may be found in~\cite{vajnovszki2013lehmer}. 

We propose next our generalization of Lehmer codes to partial rankings. Recall that the $x$-th entry in the Lehmer code of a permutation $\sigma$ is the number of elements with index smaller than $x$ that are ranked lower than $x$ in $\sigma$~\eqref{DefLC}. For a partial ranking, in addition to $\bfc_{\sigma}$, we use another code that takes into account ties according to:
 \begin{align} \label{DefPLC}
 \bfc'_{\sigma}(x)=|\{y\in [n]: y<x,\sigma(y)\geq\sigma(x)\}|.
 \end{align}
Clearly, $\bfc'_{\sigma}(x)\geq \bfc_{\sigma}(x)$ for all $x\in[n]$. It is straightforward to see that using $\bfc_{\sigma}(x)$ and $\bfc'_{\sigma}(x)$, one may recover the original partial ranking $\sigma$. In fact, we prove next that the linear-time Lehmer encoding and decoding algorithms may be used to encode and decode $\bfc_{\sigma}$ and $\bfc'_{\sigma}$ in linear time as well.

%As aforementioned, for a partial ranking, we are to use two types of Lehmer codes, $\bfc_{\sigma}$ \eqref{DefLC} and $\bfc_{\sigma}'$ \eqref{DefPLC}. Before going into the aggregation algorithm, we first introduce a simple algorithm to compute both $\bfc_{\sigma}$ and $\bfc_{\sigma}'$ based on the encoding algorithm towards Lehmer codes of permutations. Since the encoding algorithm for permutations can be computed with $(n)$ complexity, this algorithm is able to keep the same order of complexity. 
%\textcolor{blue}{The following paragraph has been revised to achieve a simpler Algorithm 1.}
Given a partial ranking $\sigma$, we may break the ties in each bucket to arrive at a permutation $\sigma'$ as follows: 
For $x,y\in S$, if $\sigma(x)=\sigma(y)$, 
\begin{align}\label{breakties}
\sigma'(x)<\sigma'(y)\;\text{if and only if}\;x<y.
\end{align}

We observe that the entries of the Lehmer codes of $\sigma$ and $\sigma'$ satisfy the following relationships for all $i\in[n]$:
\begin{align*}
&\quad\bfc_{\sigma}'(x) =\bfc_{\sigma'}(x)+ \text{IN}_{x}-1,\\
&\quad\bfc_{\sigma}(x) =\bfc_{\sigma'}(x), 
\end{align*} 
where $\text{IN}_x=|\{y\in [n] \cap \mathcal{B}_{\sigma(x)}: y \leq x\}|.$ An example illustrating these concepts is given below. %supposing $S=\{x_1,x_2,...,x_n\}$ and $f(x_i)=i$ and suppose $\sigma=(\{x_1,x_2,x_6\}\{x_3,x_4,x_7\}\{x_5,x_8,x_9\})$,  via breaking ties, we have
%\begin{center}
%\vspace*{.05in}
%       \begin{tabular}{c@{\hspace*{.6in}}c}
%  	       $\sigma'$                 & $\bfc_{\sigma'}$\\
%     $x_2$ $x_1$ $x_6$ $x_4$ $x_3$ $x_7$ $x_5$ $x_9$ $x_8$ &1 0 1 0 3 1 0 1
%      \end{tabular}
%\vspace*{.05in}
%\end{center}
%\begin{center}
%\vspace*{.05in}
%       \begin{tabular}{c@{\hspace*{.6in}}c}
%  	       $\text{IN}_i$  & $\text{LC}_i$\\
%      1 0 1 0 2 2 1 2 & 1 0 1 0 0 0 0 1
%      \end{tabular}
%\vspace*{.05in}
%\end{center}
\begin{center}
\vspace*{.05in}
       \begin{tabular}{c@{\hspace*{.2in}}ccccccccc}
       	   %$S$ & $x_1$&$x_2$& $x_3$&$x_4$&$x_5$&$x_6$&$x_7$&$x_8$&$x_9$ \\
		   $e$ & 1&2&3&4&5&6&7&8&9\\
  	       $\sigma$  &1&1&2&2&3&1&2&3&3 \\
	       $\sigma'$  &1&2&4&5&7&3&6&8&9 \\
	       $\bfc_{\sigma'}$   & 0&0&0&0&0&3&1&0&0\\
	       $\text{IN}$ &1&2&1&2&1&3&3&2&3 \\ 
	       $\bfc_{\sigma}$  &0&0&0&0&0&3&1&0&0 \\
	       $\bfc_{\sigma}'$  &0&1&0&1&0&5&3&1&2 
	       %$\bfc_\sigma$ 0\\
     %$x_2$ $x_1$ $x_6$ $x_4$ $x_3$ $x_7$ $x_5$ $x_9$ $x_8$ &1 0 1 0 3 1 0 1
      \end{tabular}
\vspace*{.05in}
\end{center}

Note that $\text{IN}_x$, as well as $\bfc_{\sigma}$ and $\bfc_{\sigma}'$ may be computed in linear time. The encoding procedure is outlined in Algorithm 1. 

\begin{table}[htb]
\centering
\begin{tabular}{l}
\hline 
\label{LCforPR}
\textbf{Algorithm 1: } \\\textbf{Lehmer encoder for partial rankings} \\
\textbf{Input:} a partial ranking $\sigma$; \\
\ 1: Set $N$ to be the number of buckets in $\sigma$;\\
\ 2: Initialize $\text{IN}=(0,0,...,0)\in \mathbb{N}^n$\\
\quad\quad and $\text{BucketSize}=(0,0,...,0) \in \mathbb{N}^N$; \\
\ 3: \textbf{For $x$ from $1$ to $n$ do}\\
\ 4: \quad $\text{BucketSize}(\sigma(x)) ++;$ \\ 
\ 5: \quad $\text{IN}(x)\leftarrow \text{BucketSize}(\sigma(x))$; \\
\ 6: Break ties of $\sigma$ to get $\sigma'$ according to \eqref{breakties};  \\
\ 7: $\bfc_{\sigma'}\leftarrow \text{Lehmer code of } \sigma'$; \\
\textbf{Output:} Output $\bfc_{\sigma}=\bfc_{\sigma'}$, $\bfc_{\sigma}'=\bfc_{\sigma}+\text{IN}-\mathbf{1}$;\\
\hline
\end{tabular}
\end{table}
%}

\section{Aggregation Algorithms} \label{sec:algorithms}
Assume that we have to aggregate a set of $m$ rankings, denoted by $\Sigma=(\sigma_1, \sigma_2, \dots, \sigma_m), \, \sigma_k\in \Sn, \, 1 \leq k\leq m$. Aggregation may be performed via the distance-based Kemeny-Young model, in which one seeks a ranking $\sigma$ that minimizes the cumulative Kendall $\tau$ (Kemeny) distance $d_{\tau}$ ($d_{K}$) from the set $\Sigma$, formally defined as:
%$$
%\sigma^{*} = \arg \min_{\sigma} d_\tau(\Sigma, \sigma),\;\text{where}\;d_\tau(\Sigma, \sigma)= \sum_{i=1}^m d_\tau(\sigma_i, \sigma).
%$$
$$
D(\Sigma, \sigma)= \sum_{i=1}^m d_\tau(\sigma_i, \sigma).
$$
Note that when the set $\Sigma$ comprises permutations only, $\sigma$ is required to be a permutation; if $\Sigma$ comprises partial rankings, we allow the output to be either a permutation or a partial ranking.

The LCA procedure under the Kendall $\tau$ distance is described in Algorithm 2. 
\begin{table}[htb]
\centering
\begin{tabular}{l}
\hline
\label{alg:LCaggregation}
\textbf{\normalsize Algorithm 2: The LCA Method (Permutations)}\\ 
%\textbf{Aggregation method for permutations}\\
\textbf{Input:} $\Sigma=\{\sigma_1,\sigma_2,...,\sigma_m\}$, 
where $\sigma_i \in \mathbb{S}_n, \, i \in [n]$.\\
\ 1: Compute the Lehmer codewords $\bfc_{\sigma_j}$ for all $\sigma_j\in\Sigma.$ \\
\ 2: Compute the median/mode of the coordinates: \\
\hspace{0.12in} $\hat{\bfc}(i) = {\rm median/mode}\Big( \bfc_{\sigma_1}(i), \bfc_{\sigma_2}(i), \dots, \bfc_{\sigma_m}(i)\Big).$ \\
\ 3: Compute $\hat{\sigma}$, the inverse Lehmer code of $\hat{\bfc}$.  \\
\textbf{Output:} Output $\hat{\sigma}.$\\
\hline
\end{tabular}
\end{table}
Note that each step of the algorithm may be executed in parallel. If no parallelization is used, the 
first step requires $O(mn)$ time, given that the Lehmer codes may be computed in $O(n)$ time~\cite{marevs2007linear,myrvold2001ranking}. If parallelization on $\Sigma$ is used instead, the time reduces to $O(m+n)$. Similarly, without parallelization the second step requires 
$O(mn)$ time, while coordinate parallelization reduces this time to $O(m)$. This third step requires $O(n)$ computations. Hence, the overall complexity of the algorithm is either $O(mn)$ or $O(m+n)$, depending on parallelization being used or not. 

For permutations, the aggregation procedure may be viewed as specialized voting: The ranking $\sigma_k$ casts a vote to rank $x$ at position $x-\bfc_{\sigma_k}(x),$ for the case that only elements $\leq x$ are considered (A vote corresponds to some score confined to $[0,1]$). However, when $\sigma_k$ is a partial ranking involving ties, the vote should account for all possible placements between $x-\bfc_{\sigma}'(x)$ and $x-\bfc_{\sigma}(x)$. More precisely, suppose that the vote cast by $\sigma_k$ to place element $x$ in position $y\in [x]$ is denoted by $v_{k\rightarrow x}(y)$. Then, one should have 
%$v_{k\rightarrow i}(j)>0$ 
%It is reasonable to set the nonzero $v_{k\rightarrow i}(j)$ to be 1 according to our results in Section V but it also admits fractional votes, i.e., $0<v_{k\rightarrow i}(j)<1$. In this paper, we set $v_{k\rightarrow i}(j)$ to be a positive constant for all possible positions. 
%Concretely, if $j\in[i-\bfc_{\sigma}'(i), i-\bfc_{\sigma}(i)]$
\begin{equation} \label{votevalue}
v_{k\rightarrow x}(y)=\left\{\begin{array}{cl}
1, & \text{for the mode,}  \\
\frac{1}{\bfc_{\sigma}'(x)-\bfc_{\sigma}(x)+1}, & \text{for the median,}
\end{array}\right.
\end{equation}
if and only if $y\in[x-\bfc_{\sigma}'(x), x-\bfc_{\sigma}(x)]$, and zero otherwise.
Note that when the mode is used, the ``positive votes'' are all equal to one, while when the median is used, a vote counts only a fractional value dictated by the length of the ``ranking interval''. 

Next, we use $V_{x}(y)=\sum_{k=1}^{m}v_{k\rightarrow x}(y)$ to denote the total voting score element $x$ received to be ranked at position $y$. The inverse Lehmer code of the aggregator output $\hat{\sigma}$ is computed as:
%can be similarly computed as the mode/median of  $V_{i}(j)$ over all $j\in[i]$ as the following 
\begin{align} \label{modePR}
\text{mode:}\quad \hat{\bfc}(x)& = \arg_{y\in[x]}\max V_x(y) -1 , \\ \label{medianPR}
\text{median:}\quad \hat{\bfc}(x)& = \min\{k:\frac{\sum_{y=1}^{k}V_{x}(y)}{m}\geq 1/2\}-1. \notag
\end{align}
%Based on the estimation \eqref{modePR} or \eqref{medianPR}, we conclude the following Algorithm~3 to aggregate partial rankings.
%\begin{table}[htb]
%\centering
%\begin{tabular}{l}
%\hline
%\label{alg:LCaggregationPR}
%\textbf{Algorithm 3:} \\
%\textbf{Aggregation method for partial rankings}\\
%\textbf{Input:} $\Sigma=\{\sigma_1,\sigma_2,...,\sigma_m\}$; \\
%\ 1: Compute $\bfc_{\sigma_j}, \bfc_{\sigma_j}'$ for all $\sigma_j\in\Sigma$ via Algorithm 2. \\ 
%\ 2: Compute the mode/median of the Lehmer codes $\bfc_{\hat{\sigma}'}$, \\
%\quad\quad according to  \eqref{modePR} and \eqref{medianPR} for all $i\in[n-1]$. \\
%\ 3: Find $\hat{\sigma}'$, the inverse Lehmer transform of $\bfc_{\hat{\sigma}'}$.  \\
%\textbf{Output:} Output $\hat{\sigma}$;\\
%\hline
%\end{tabular}
%\end{table}
To compute the values $V_{x}(y)$ for all $y \in [x]$, the LCA algorithm requires $O(mx)$ time, which yields an overall aggregation complexity of $O(mn^2)$ when no parallelization is used. This complexity is reduced to $O(m+n^2)$ for the parallel implementation. Note that the evaluations of the $V$ functions may be performed in a simple iterative manner provided that the votes $v_{k\rightarrow x}(y)$ are positive constants, leading to a reduction in the overall complexity of this step to $O(mn+n^2)$ when no parallelization is used. Relevant details regarding the iterative procedure may be found in Appendix~\ref{supplementalg}. 
%In our setting, given for one certain $j\in[m]$, $v_{k\rightarrow i}(j)$ is a positive constant $v_{k\rightarrow i}$ \eqref{votevalue}, when $j\in[i-\bfc_{\sigma}'(i), i-\bfc_{\sigma}(i)]$, an efficient algorithm with $(m)+(i)$ complexity is shown in~Algorithm 4. Therefore, in this case, the Algorithm 3 is with complexity $(mn+n^2)$.
%Considering the votes cast by one certain $\sigma_i$ constitute a continuous interval,
%Although $\hat{\bfc}$ can be computed according to \eqref{modePR} or \eqref{medianPR}, the computation of $V_{i}(k)$ for all $k\in[i]$ is with complexity $(mi)$, which yields $(mn^2)$  overall. 
%online achieve $(mn^2)$ aggregation algorithm to get $\hat{\bfc}(i)$. 
%\begin{table}[htb]
%\centering
%\begin{tabular}{l}
%\hline
%\label{alg:LCaggregationPRvote}
%\textbf{Algorithm 4: }\\ \textbf{Computing $\{V_{i}(j)\}_{j\in[i]}$ for constant votes}\\
%\textbf{Input:} $\bfc_{\sigma_k}, \bfc_{\sigma_k}'$, constant votes $v_{k\rightarrow i}$, $\forall\;k\in[m]$; \\
%\ 1: Initialize $V_{i}(j)=0$ for all $j\in[i]$; \\
%\ 2: \textbf{For $k$ from $1$ to $m$ do}\\
%\ 3: \quad $V_{i}(i-\bfc_{\sigma_j}'(i))=V_{i}(i-\bfc_{\sigma_k}'(i))+v_{k\rightarrow i}$; \\
%\ 4: \quad $V_{i}(i-\bfc_{\sigma_j}(i))=V_{i}(i-\bfc_{\sigma_k}(i))-v_{k\rightarrow i}$; \\
%\ 5: \textbf{For $k$ from $2$ to $i$ do}\\
%\ 6: \quad $V_{i}(j)=V_{i}(j-1)+V_{i}(j)$; \\
%\ 7: \textbf{Output:} Output $V_{i}(j)$;\\
%\hline
%\end{tabular}
%\end{table}

Note that the output $\hat{\sigma}$ of Algorithm 2 is a permutation. %, provided that the input set $\Sigma$ consists of permutations. 
To generate a partial ranking that minimizes the Kemeny distance while being consistent\footnote{We say that two partial rankings $\sigma,\,\pi$ are \emph{consistent} if for any two elements $x,\,y$, $\sigma(x)<\sigma(y)$ if and only if $\pi(x)\leq \pi(y)$ and vise versa.} with $\hat{\sigma}$, one can use a $O(mn^2+n^3)$-time algorithm outlined in  Appendix~\ref{supplementalg}. Alternatively, the following simple greedy method always produces practically good partial rankings with $O(mn)$ complexity: Scan the elements in the output permutation from highest ($j=1$) to lowest rank ($j=n-1$) and decide to put $\hat{\sigma}^{-1}(j+1)$ and $\hat{\sigma}^{-1}(j)$ in the same bucket or not based on which of the two choices offers smaller Kemeny distance with respect to the subset $\{\hat{\sigma}^{-1}(1),...,\hat{\sigma}^{-1}(j)\}$.

\textbf{Discussion.} In what follows, we briefly outline the similarities and differences between the LCA method and existing positional as well as InsertionSort based aggregation methods. Positional methods are a class of aggregation algorithms that seek to output a ranking in which the position of each element is ``close'' to the position of the element in $\Sigma$. One example of a positional method is Borda's algorithm, which is known to produce a $5$-approximation to the Kemeny-Young problem for permutations~\cite{coppersmith2006ordering}. 
%\textcolor{red}{I am considering whether we should add the statement: The complexity of Borda algorithm is $O(mn)+O(nlog n)$}. 
Another method is the Spearman footrule aggregation method which seeks to find a permutation that minimizes the sum of the Spearman footrule distance between the output and each ranking in $\Sigma$. As already mentioned, the latter method produces a $2$-approximation for the Kendall $\tau$ aggregate for both permutations and partial ranking. LCA also falls under the category of positional methods, but the positions on which scoring is performed are highly specialized by the Lehmer code. And although it appears hard to prove worst-case performance guarantees for the method, statistical analysis on particular ranking models shows that it can recover the correct results with small sample complexity. It also offers significant reductions in computational time compared to the Spearman footrule method, which reduces to solving a weighted bipartite matching problem and hence has complexity at least $O(mn^2+n^3)$~\cite{dwork2001rankw}, or $O(mn)$ when implemented in MapReduce~\cite{kambatla2012efficient}. 
%\textcolor{red}{As will be seen in Section~\ref{sec:simulations}, LCA outperforms Borda's method for the Mallows model for a number of different parameter settings.}

A related type of aggregation is based on InsertionSort~\cite{dwork2001rank,dwork2001rankw}. In each iteration, an element is randomly chosen to be inserted into the sequence containing the already sorted elements. The position of the insertion is selected as follows. Assume that the elements are inserted according to the identity order $e=(1,2,\ldots,n)$ so that at iteration $t$, element $t$ is chosen to be inserted into some previously constructed ranking over $[t-1]$. %Let $S_{t-1}=[t-1]$ denote the set of elements that have been ranked before including $t$. 
Let $S_{t-1}=[t-1]$ and the symbol $t$ is inserted into the ranking over $S_{t-1}$ to arrive at $\sigma_{S_t}$, the ranking available after iteration $t$. If $t$ is inserted between two adjacent elements $\sigma_{S_{t-1}}^{-1}(i-1)$ and $\sigma_{S_{t-1}}^{-1}(i)$, then one should have $\sigma_{S_t}(x)=\sigma_{S_{t-1}}(x)$ when $\sigma_{S_{t-1}}(x)\leq i-1$, $\sigma_{S_t}(x)=\sigma_{S_{t-1}}(x-1)+1$ when $\sigma_{S_{t-1}}(x)\geq i$ and $\sigma_{S_t}(t)=i$. Let $\sigma_{S_t}(t)$ denote the rank assigned to element $t$ over $S_t$, the choice of which may vary from method to method. 
%Different InsertionSort algorithms have different ways to determine $\pi_{S_t}^{-1}(t)$, 
The authors of~\cite{dwork2001rank} proposed setting $\sigma_{S_t}(t)$ to
\begin{align*}
\max &\left\{i\in[t-1]: \sum_{k\in[m]}1_{\sigma_{k}(t)<\sigma_{k}(\sigma_{S_{t-1}}^{-1}(i))}<\frac{m}{2}\right\},
\end{align*}
or $t$ when the above set is empty. This insertion rule does not ensure a constant approximation guarantee in the worst case (It has an expected worst-case performance guarantee of $\Omega(n)$), although it leads to a Locally Kemeny optimal solution. 
%Furthermore, calculating $s_t$ in each iteration may require $(mn^2)$ operations.  
%Now, we begin illustrating that the aggregation based on Lehmer Code that can be understood as the Insertion-Sort procedure by specially selecting the inserting place via Positional method. 

We next describe how the LCA method may be viewed as an InsertionSort method with a special choice of $\sigma_{S_t}(t)$. Consider the permutation LCA method of Algorithm 2, and focus on estimating the $t$-th coordinate of the Lehmer code $\hat{\bfc}(t)$ (step 2) and the inverse Lehmer code via insertion (step 3) simultaneously. Once $\hat{\bfc}(t)$ is generated, it's corresponding inverse Lehmer transform may be viewed as the operation of placing the element $t$ at position $(t-\hat{\bfc}(t))$ over $S_t$. In other words, inverting the incomplete ranking reduces to setting $\sigma_{S_t}(t)=(t-\hat{\bfc}(t))$, where $\sigma_{S_t}(t)$ essentially equals the mode or median of the positions of $t$ in the 
rankings of $\Sigma$, projected onto $S_{t}$. 
%Therefore, the ranking aggregation method based on Lehmer code is a way combining insertion-sort procedure with positional selection. 
The same is true of partial rankings, with the only difference being that the selection of $\sigma_{S_t}(t)$ has to be changed because of ties between elements.

\section{Analysis of the Mallows Model} \label{sec:analysis}

We provide next a theoretical performance analysis of the LCA algorithm under the assumption that the rankings are generated according to the Mallows and generalized Mallows Model. In the Mallows model MM$(\sigma_0,\phi)$ with parameters $\sigma_0$ and $\phi$, $\sigma_0$ denotes the centroid ranking and $\phi\in(0,1]$ determines the variance of the ranking with respect to $\sigma_0$. The probability of a permutation $\sigma$ is proportional to $\phi^{d_\tau(\sigma_0, \sigma)}$. For partial rankings, we assume that the samples are generated from a generalized Mallows Model (GMM) whose centroid is allowed to be a partial ranking and where the distance is the Kemeny $d_k$, rather than the Kendall $\tau$ distance $d_{\tau}$.

Our analysis is based on the premise that given a sufficiently large number of samples (permutations), one expects the ranking obtained by a good aggregation algorithm to be equal to the centroid $\sigma_0$ with high probability. Alternative methods to analytically test the quality of an aggregation algorithm are to perform a worst-case analysis, which for the LCA method appears hard, or to perform a simulation-based analysis which produces a comparison of the objective function values for the Kemeny-Young problem given different aggregation methods. We report on the latter study in the section to follow.

To ease the notational burden, we henceforth use $\phi_{s:t}\triangleq\sum_{k=s}^t\phi^k$ in all subsequent results and derivations. Detailed proofs are relegated to the appendix. One of our main theoretical result is the following.

\begin{theorem} \label{Thm1}
Assume that $\Sigma=\{\sigma_1,\sigma_2,...,\sigma_m\}$, where $\sigma_k\;\stackrel{\text{i.i.d}}{\sim}\;$ MM$(\sigma_0,\phi),$  $k\in[m],$ are $m$ i.i.d. samples of the given Mallows model. If $\phi+\phi^2<1+\phi^n$ and $m\geq c\log\frac{n^2}{2\delta}$ with $c=\frac{2(1+q)^2}{(1-q)^4}$ and $q=\frac{\phi_{1:n-1}}{1+\phi_{3:n}}$, then the output ranking of Algorithm 2 under the mode rule equals $\sigma_0$ with probability at least $1-\delta$.  
\label{thm-mode}
\end{theorem}
%Here, and in all subsequent results of the paper, we use $\sigma=(\sigma(1),\ldots,\sigma(n))$ to denote element-wise rankings in which $\sigma(i)$ equals the element at position $i$. 

The idea behind the proof is to view the LCA procedure as an InsertionSort method, in which the probability of the event that the selected position is incorrect with respect to $\sigma_0$ is very small for sufficiently large $m$. Based on the lemma that follows (Lemma \ref{Lemma1}), one may show that if $\phi$ satisfies $\phi+\phi^2<1+\phi^n$, the most probable position of an element in a ranking $\sigma\sim\;$ MM$(\sigma_0,\phi)$ corresponds to its rank in the centroid $\sigma_0$. Given enough samples, one can estimate the rank of an element in the centroid by directly using the mode of the rank of the element in the drawn samples. 
\begin{lemma} \label{Lemma1}
Let $\sigma\sim\;$ MM$(\sigma_0,\phi)$. Consider an element $u$. Then, the following two statements describe the distribution of $\sigma(u)$: 
\begin{align*}
%\item $f(i\rightarrow j)$ is maximal when $j=i$.
\text{1)}\;&\frac{\prob{[\sigma(u)=j+1]}}{\prob{[\sigma(u)=j]}}\in[\phi, \frac{\phi_{1:n-1}}{1+\phi_{3:n}}]\; \text{when} \; \sigma_0(u) \leq j < n. \\
\text{2)}\;&\frac{\prob{[\sigma(u)=j-1]}}{\prob{[\sigma(u)=j]}}\in[\phi, \frac{\phi_{1:n-1}}{1+\phi_{3:n}}]\; \text{when} \;1<j\leq \sigma_0(u). 
\end{align*}
In 1), the upper bound is achieved when $\sigma_0(u)=n-1$ and $j=\sigma_0(u),$ while the lower bound is achieved when $\sigma_0(u)=1$. In 2), the upper bound is achieved when $\sigma_0(u)=2$ and $j=\sigma_0(u),$ while the lower bound is achieved when $\sigma_0(u)=n$.
\end{lemma}
%Proof is postponed to the section I of the supplementary document. 
\begin{remark} The result above may seem counterintuitive since it implies that for $\phi+\phi^2>1+\phi^n$, the probability of ranking some element $u$ at a position different from its position in $\sigma_0$ is larger than the probability of raking it at position $\sigma_0(u)$. An easy-to-check example that shows that this indeed may be the case corresponds to $\sigma_0=(1,2,3,4)$ and $\phi=0.9$. Here, we have $\prob{[\sigma(3)=3]}=0.2559<\prob{[\sigma(3)=4]}=0.2617.$
\end{remark} 
%\textcolor{blue}{I saw your comments about the unclearness. Your saying makes sense. The current statement is much clearer.}

Lemma~\ref{Lemma1} does not guarantee that in any single iteration the position of the element will be correct, since the ranking involves only a subset of elements. Therefore, Lemma~\ref{Lemma2}, a generalized version for the subset-projected ranking, is required for the proof. 

\begin{lemma} \label{Lemma2}
Let $\sigma\sim\;$ MM$(\sigma_0,\phi)$ and let $A \subset [n]$. Consider an element $u\in A$. Then, the following two statements describe the distribution of  $\sigma_A(u)$:
\begin{align*}
\text{1)}\;&\frac{\prob{[\sigma_A(u)=j+1]}}{\prob{[\sigma_A(u)=j]}}\leq\max_{l\in [0,n-|A|]}\frac{\phi+\phi^{l}\phi_{2:n-l-1}}{1+\phi^{2l}\phi_{3:n-l}}\; \\&\text{when} \; |A|>j\geq \sigma_{0,A}(u).  \\
\text{2)}\;&\frac{\prob{[\sigma_A(u)=j-1]}}{\prob{[\sigma_A(u)=j]}}\leq\max_{l\in [0,n-|A|]}\frac{\phi+\phi^{l}\phi_{2:n-l-1}}{1+\phi^{2l}\phi_{3:n-l}}\; \\&\text{when} \; 1<j\leq \sigma_{0,A}(u). 
\end{align*}
Observe that the conditions that allow one to achieve the upper bound in Lemma~\ref{Lemma1} also ensure that the upper bounds are achieved in Lemma~\ref{Lemma2}. Moreover, when $\phi+\phi^2<1+\phi^n$, the right hand sides are $\leq \frac{\phi_{1:n-1}}{1+\phi_{3:n}}$.
\end{lemma}
%Proof is postponed to the section II of the supplementary document. 
%In the inserting iteration $t$, the element $t$ is to be ranked over $S_t$. Since $\sigma_k\sim$MM$(\sigma_0,\phi)$, let $X_k(j)=1_{\{\sigma_{k,S_t}(t)=j\}}$ be the random variable indicating whether $S_t$-projected $\sigma_k$ ranks element $t$ in $j$th or not. If $\phi+\phi^2<1+\phi^n$, by using the Lemma 5.3 and Hoeffding's inequality, we have $\prob{[\sum_{k\in[m]}[X_k(\sigma_{0,S_t}(t))-X_k(j)]\leq 0]}$ is exponentially small in terms of $m$ when $j\neq \sigma_{0,S_t}(t)$. Take union bound over all possible position $j$'s and different iterations and we can achieve Theorem~\ref{Thm1}. A more detailed analysis is shown in the section V of the supplementary document.

The next result establishes the performance guarantees for the LCA algorithm with the median operation. 
\begin{theorem} \label{Thm2}
Assume that $\Sigma=\{\sigma_1,\sigma_2,...,\sigma_m\}$, where $\sigma_k\;\stackrel{\text{i.i.d}}{\sim}\;$ MM$(\sigma_0,\phi),$ $k\in[m]$. If $\phi<0.5$ and $m\geq c\log\frac{2n}{\delta},$ where $c=\frac{2}{(1-2\phi)^2}$, then the output of Algorithm 2 under the median operation equals $\sigma_0$ with probability at least $1-\delta$.  
 \label{thm-median}
\end{theorem}

The proof follows by observing that if the median of the Lehmer code $c_{\sigma_{k}}(t)$ over all $k\in[m]$ converges to $t-\sigma_{0,S_t}(t)$ as $m\rightarrow\infty$, then each $\sigma_k$ should have $\prob{[\sigma_{k,S_t}(t)>\sigma_{0,S_t}(t)]},\prob{[\sigma_{k,S_t}(t)<\sigma_{0,S_t}(t)]}<1/2$. According to the following Lemma, in this case, one needs $\phi<0.5$.

\begin{lemma} \label{Lemma3}
Let $\sigma\sim\;$ MM$(\sigma_0,\phi)$ and let $A \subseteq [n]$. %Let element $u \in A$. 
For any $u \in A$, the following two bounds hold:
\begin{align*}
&\text{1)}\;\prob{[\sigma_A(u)>\sigma_{0,A}(u)]}\leq \frac{\phi_{1:(|A|-\sigma_{0,A}(u))}}{\phi_{0:(|A|-\sigma_{0,A}(u))}}<\phi,\\
&\text{2)}\;\prob{[\sigma_A(u)<\sigma_{0,A}(u)]} \leq \frac{\phi_{1:\sigma_{0,A}(u)}}{\phi_{0:\sigma_{0,A}(u)}}<\phi.
%\sum_{j=1}^{i-1} f_A(i\rightarrow j) \leq  \phi_{1:i}\sum_{j=i}^{|A|}f_A(i\rightarrow j) 
\end{align*}
The inequality 1) is met for $A=S$ and $\sigma_{0}(u)=1,$ while the inequality 2) is met for $A=S$ and $\sigma_{0}(u)=n$.
\end{lemma}
%Let GMM$(\sigma_0,\phi)$ denote the generalized Mallows Model, where $\sigma_0$ is the centroid partial ranking and $\phi\in(0,1]$ indicates the ``variance'' of this model. Mathematically, the probability to a partial ranking $\sigma$ is proportional to $\phi^{d_\tau(\sigma_0, \sigma)}$, where $d_\tau(\sigma_0, \sigma)$ is defined for partial rankings according to equation~\eqref{partialmetric}. 
We now turn our attention to partial rankings and prove the following extension of the previous result for the GMM, under the LCA algorithm that uses the median of coordinate values. Note that the output of Algorithm 2 is essentially a permutation, although it may be transformed into a partial ranking via the bucketing method described in Section 2. 
%So for the partial ranking case, we expect the output of Algorithm 2 to be a full ranking generated from breaking ties of the centroid partial ranking with high probability. . 
\begin{theorem} \label{Thm4}
Assume that $\Sigma=\{\sigma_1,\sigma_2,...,\sigma_m\}$, where $\sigma_k\;\stackrel{\text{i.i.d}}{\sim}\;$GMM$(\sigma_0,\phi),$ $k\in[m]$. 
%Let $\Sigma_0$ be the set of all possible permutations obtained by breaking the ties in $\sigma_0$. 
If $\phi+\phi^{1/2}<1$ and $m\geq c\log\frac{2n}{\delta}$ with $c=\frac{2}{(1-2q')^2},$ where $q'=1-\frac{1}{2}\phi^{1/2}-\frac{1}{2}\phi$, then the output ranking of the LCA algorithm (see Appendix~\ref{ProofThm1}) under the median operation is in $\Sigma_0$ with probability at least $1-\delta$. Here, $\Sigma_0$ denotes the set of permutations generated by breaking ties in $\sigma_0$.
\label{thm-mode-par}
\end{theorem}
%\textcolor{blue}{The output of Algorithm 2 is essentially a full ranking. So we cannot expect the output to be $\sigma_0$. In fact, the output is in the set of full rankings generated from breaking ties in $\sigma_0$. So I revised as above. }

%We only prove that the output of the first stage (step 1-3) is a permutation obtained by breaking ties in $\sigma_0$ with high probability. The proof of the second stage is as same as Theorem~\ref{Thm3}.
The proof of this theorem relies on showing that the InsertionSort procedure places elements in their correct position with high probability. If the median is used for partial ranking aggregation, one vote is uniformly distributed amongst all possible positions in the range given by~\eqref{votevalue}. To ensure that the output permutation is in $\Sigma_0$, we need to guarantee that the median of the positions of the votes for $t$ over $S_t$ is in $[l_{\sigma_0,S_t}(t), r_{\sigma_0,S_t}(t)]$ for large enough $m$ (as in this case, $[l_{\sigma_0,S_t}(t), r_{\sigma_0,S_t}(t)]$ represents the bucket in $\sigma_0$ that contains $t$).

For a $\sigma\;\sim\;$GMM$(\sigma_0,\phi)$, let $v(j)$ be the vote that the partial ranking $\sigma$ cast for position $j$. Then, one requires that
\begin{equation}
\mathbb{E}[\sum_{k=1}^{r_{\sigma_{0,A}(u)}}v(j)]> 0.5 \; \text{ and } \; \mathbb{E}[\sum_{k=l_{\sigma_{0,A}(u)}}^{n}v(j)]> 0.5. \notag
\end{equation}
The expectations in the expressions above may be evaluated as follows (We only consider the expectation on the left because of symmetry). If the event $W=\{r_{\sigma_{S_t}(t)}\leq r_{\sigma_{0,S_t}(t)}\}$ occurs, then the vote of $\sigma$ that contributes to the sum equals $1$. If the event $Q=\cup_{j=1}^{n-r_{\sigma_{0,S_t}(t)}}Q_j$, where $Q_j=\{r_{\sigma_{S_t}(t)}=j+r_{\sigma_{0,S_t}(t)}, \;l_{\sigma_{S_t}(t)}\leq r_{\sigma_{0,S_t}(t)}\}$ occurs, then the vote that $\sigma$ contributes to the sum equals $V_j=\frac{r_{\sigma_{0,S_t}(t)}-l_{\sigma_{S_t}(t)}+1}{r_{\sigma_{S_t}(t)}-l_{\sigma_{S_t}(t)}+1}.$ Therefore, we have 
\begin{equation} 
\mathbb{E}{[\sum_{k=1}^{r_{\sigma_{0,{S_t}}(t)}}v(k)]}= \prob{[W]} + \sum_{j=1}^{n-r_{\sigma_0}(u)}V_j\prob{[Q_j]}. \label{partialmedian3}
\end{equation}
%%%%%%%%%%%%%%%%%%%%%%%%%%%%%%%
The following lemma describes a lower bound for~\eqref{partialmedian3}. 
%By using \eqref{partialmedian4}, we know that the lower bound of the LHS of \eqref{partialmedian1} is achieved when the element $u$ is ranked smallest in set $A$, which is similar to the permutation case. However, because of ties, the further lower bound of the LHS of \eqref{partialmedian1} is not easily computable. With carefully partitioning the events into different cases, we can find a lower bound as shown in \eqref{partialmedian5}. 
\begin{lemma} \label{Lemma6}
Let $\sigma\sim\;$GMM$(\sigma_0,\phi)$ and let $A \subseteq [n]$ be such that it contains a predefined element $u$. Let $A'=A-\{x\in A: x\neq u, \sigma_{0,A}(x)\leq\sigma_{0,A}(u) \}$. Define 
\begin{align}
&W=\{r_{\sigma_{A}(u)}\leq r_{\sigma_{0,A}(u)}\}, \nonumber \\ 
&Q_j=\{r_{\sigma_{A}(u)}=j+r_{\sigma_{0,A}(u)},\;l_{\sigma_A(u)}\leq r_{\sigma_{0,A}(u)}\}, \nonumber \\ 
&W'=\{r_{\sigma_{A'}(u)}\leq r_{\sigma_{0,A'}(u)}\}, \nonumber \\ 
&Q_j'=\{r_{\sigma_{A'}(u)}=j+r_{\sigma_{0,A'}(u)},\;l_{\sigma_A'(u)}\leq r_{\sigma_{0,A'}(u)}\}. \notag
\end{align}
Then, one can prove that
\begin{eqnarray}
&\prob{[W]} + \sum_{j=1}^{|A|-r_{\sigma_{0,A}}(u)}V_j\prob{[Q_j]}\nonumber\\
\geq& \prob{[W']} + \sum_{j=1}^{|A'|-r_{\sigma_{0,A'}}(u)}\frac{1}{j+1}V_j\prob{[Q_j']} \label{partialmedian4} \nonumber \\
\geq& 1-\frac{1}{2}\phi^{1/2}-\frac{1}{2}\phi. \nonumber 
\end{eqnarray} 
\end{lemma}
If $\phi+\phi^{1/2}<1$, the lower bound above exceeds $1/2$. Theorem~\ref{Thm4} then follows using the union bound and Hoeffding's inequality.

\section{Performance Evaluation} \label{sec:simulations}

We next evaluate the performance of the LCA algorithms via experimental methods and compare it to that of other rank aggregation methods using both synthetic and real datasets. For comparative analysis, we choose the Fas-Pivot and FasLP-Pivot (LP) methods~\cite{ailon2008aggregating}, InsertionSort with Comparison (InsertionComp) from~\cite{dwork2001rank}, and the optimal Spearman Footrule distance aggregator (Spearman)~\cite{diaconis1977spearman}. For the randomized algorithms Fas-Pivot and FasLP-Pivot, the pivot in each iteration is chosen randomly. For InsertionSort with Comparison, the insertion order of the elements is also chosen randomly. Furthermore, for all three methods, the procedure is executed five times, and the best solution is selected. For Fas-Pivot and FasLP-Pivot, we chose the better result of Pick-A-Perm and the given method, as suggested in~\cite{ailon2008aggregating}. 
%For synthetic data, we only present aggregation results for the Mallows Model to generate $\Sigma$ and will also use the most probable ranking (MostProb) for comparison. 
%The probability of a ranking $\sigma\in\mathbb{S}^n$ from Plackett-Luce Model with a parametric vector $v=(v_1,v_2,...,v_n)$ is equal to \eqref{PLM}
%\begin{align}\label{PLM}
%\prob{[\sigma]}=\prod_{i=1}^{n-1}\frac{v_{\sigma^{-1}(i)}}{\sum_{j=i}^{n}v_{\sigma^{-1}(j)}}.
%\end{align}

In the context of synthetic data, we only present results for the Mallows model in which the number of ranked items equals $n=10,$ and the number of rankings equals $m=50$. The variance parameter was chosen according to $\phi=e^{-\lambda}$, where $\lambda$ is allowed to vary in $[0,1]$.  For each parameter setting, we ran $50$ independent simulations and computed the average cumulative Kendall $\tau$ distance (normalized by $m$) between the output ranking and $\Sigma$, given as $D_{av}=\frac{D(\sigma,\Sigma)}{m}$. We then normalized the $D_{av}$ value of each algorithm by that of FasLP-Pivot, since FasLP-Pivot always offered the best performance. The results are depicted in Fig.~\ref{fig:MM}. Note that we used MostProb to describe the most probable ranking, which is the centroid for the Mallows Model.

\begin{figure}[htb]
\centering
  \centerline{\includegraphics[width=9cm]{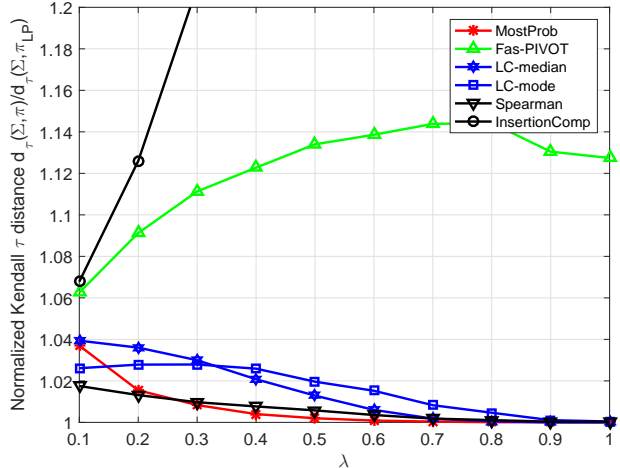}}
  \caption{The normalized Kendall $\tau$ Distance vs the parameter $\lambda$ of the Mallows Model.}
\label{fig:MM}
\end{figure}
%\begin{figure}[htb]
%
%\centering
%  \centerline{\includegraphics[width=7cm]{PL_model_rel_new.eps}}
%  \caption{Nomalized Kendall-tau Distance vs parameter $\tau$ (Plackett-Luce Model)}
%\label{fig:PL}
%\end{figure}
%
%\begin{figure}[htb]
%
%\centering
%  \centerline{\includegraphics[width=8.5cm]{sushi.eps}}
%  \caption{Average Kendall-tau Distance vs The number of rankings $m$ (sushi preference data sets)}
%\label{fig:sushi}
%\end{figure}
Note that for parameter values $\lambda \geq 0.6$ LCA algorithms perform almost identically to the best aggregation method, the LP-based pivoting scheme. For smaller values of $\lambda$, small performance differences may be observed; these are compensated by the significantly smaller complexity of the LCA methods which in the parallel implementation mode is only linear in $n$ and $m$. Note that the InsertionSort Comp method performs poorly, although it ensures local Kemeny optimality. 

We also conducted experiments on a number of real-world datasets. To test the permutation LCA aggregation algorithms, we used the Sushi ranking dataset~\cite{kamishima2003nantonac} and the Jester dataset~\cite{goldberg2001eigentaste}. The Sushi dataset consists of $5000$ permutations involving $n=10$ types of sushi. The Jester dataset contains scores in the continuous interval $[-10,10]$ for $n=100$ jokes submitted by $48483$ individuals. 
We chose the scores of $14116$ individuals who rated all $100$ jokes and transformed the rating into permutations by sorting the scores. For each dataset, we tested our algorithms by randomly choosing $m$ many samples out of the complete list and by computing the average cumulative Kendall $\tau$ distance normalized by $m$ via $50$ independent tests. The results are listed in the Table~\ref{sushifull} and Table~\ref{jesterfull}.

\begin{table} 
\caption{Rank aggregator comparison for the Sushi dataset (permutations) }
\centering
\label{sushifull}
\vspace*{.05in}
       \begin{tabular}{c@{\hspace*{.15in}}|c|c|c|c|c}
       	  $m$ 		& 10 	& 50 & 200 & 1000 & 5000 \\
	  \hline
      	Fas-Pivot  		&   14.51 	& 15.98 & 16.18 & 16.38 & 16.06  \\
      	FasLP-Piovt     	&  13.59 & 15.00 & 15.33 & 15.39 & 15.39 \\
   	InsertionComp  	& 15.87    & 16.60 & 16.70 & 16.80 & 16.65 \\
	Spearman      	&  14.41 & 15.24 & 15.54 & 15.56 & 15.61 \\ 
	LC-median       	&  14.03  & 15.25  & 15.57 & 15.58 & 15.74 \\
	LC-mode        	&  14.19 & 15.33 & 15.46 & 15.47 & 15.49 \\
      \end{tabular}
\vspace*{.05in}
\end{table}

\begin{table} 
\centering
\caption{Rank aggregator comparison for the Jester dataset (permutations) }
\label{jesterfull}
\vspace*{.05in}
       \begin{tabular}{c@{\hspace*{.15in}}|c|c|c|c|c}
       	  $m$ 		& 50 	& 200 & 1000 & 5000 & 10000 \\
	  \hline
      	Fas-Pivot  		&   2102 	&  2137 & 2144 &  2127 & 2127  \\
	FasLP-Piovt     	&  1874  & 1915  & 1920 & 1922 & 1921\\
   	InsertionComp  	& 2327   & 2331 & 2337 & 2323 & 2390 \\
	Spearman      	&  1900 & 1936 & 1935 & 1937 & 1937 \\ 
	LC-median       	&  1932  & 1962  & 1965 & 1966 & 1965 \\
	LC-mode        	&  1973  & 1965 & 1962 & 1964 & 1965 \\
      \end{tabular}
\vspace*{.05in}
\end{table}

To test our partial ranking aggregation algorithms, we used the complete Jester dataset~\cite{goldberg2001eigentaste} and the Movielens dataset~\cite{harper2016movielens}. For the Jester dataset, we first rounded the scores to the nearest integer and then placed the jokes with the same integer score in the same bucket of the resulting partial ranking. We also assumed that the unrated jokes were placed in a bucket ranked lower than any other bucket of the rated jokes. The movielens dataset contains incomplete lists of scores for more than $1682$ movies rated by $943$ users. The scores are integers in $[5],$ so that many ties are present. We chose the $50$ most rated movies and $500$ users who rated these movies with largest coverage. Similarly as for the Jester dataset, we assumed that the unrated movies were tied for the last position. In each test, we used the iterative method described in Section 3 to transform permutations into partial rankings. Note that when computing the Kemeny distance between two partial rankings of~\eqref{partialmetric}, we omitted the penalty incurred by ties between unrated elements, because otherwise the iterative method would yield too many ties in the output partial ranking. More precisely, we used the following formula to assess the distance between two incomplete partial rankings~\eqref{partialmetric2}: 
\begin{align}
&d_{\tau}(\pi,\sigma)=|\{(x,y): \pi(x)>\sigma(y),\pi(x)<\sigma(y)\}|\nonumber\\
+&\frac{1}{2}|\{(x,y): \left[\pi(x)=\pi(y),\sigma(x)>\sigma(y),\;x,y\;\text{rated by}\,\pi \right] \nonumber \\
\text{or}&\;\left[\pi(x)>\pi(y),\sigma(x)=\sigma(y), \;x,y\;\text{rated by}\,\sigma\right]\}|. \label{partialmetric2}
\end{align}
%For each dataset, we test different algorithms by randomly choosing $m$ many samples and computes the average overall Kendall-$\tau$ distance (normalized by $m$) over more than 50 times independent tests. 
The results are listed in Table~\ref{jesterpartial} and Table~\ref{moviepartial}. As may be seen, the parallelizable, low-complexity LCA methods tend to offer very similar performance to that of the significantly more computationally demanding LP pivoting algorithm.
\begin{table} 
\centering
\caption{Rank aggregator comparison for the Jester dataset (partial rankings)}
\label{jesterpartial}
\vspace*{.05in}
       \begin{tabular}{c@{\hspace*{.15in}}|c|c|c|c|c}
       	  $m$ 		& 50 	& 200 & 1000 & 5000 & 10000 \\
	  \hline
      	Fas-Pivot  		&   1265 	&  1280 & 1279 &  1279 & 1281  \\
	FasLP-Piovt     	&  1264  & 1280  & 1279 & 1279 & 1281\\
   	InsertionComp  	&  1980   & 1967 & 1956 & 1949 & 1979 \\
	Spearman      	&  1272 & 1284 & 1281 & 1281 & 1282 \\ 
	LC-median       	&  1275  & 1287  & 1284 & 1283 & 1287 \\
	LC-mode        	&  1311  & 1304 & 1289 & 1283 & 1283 \\
      \end{tabular}
\vspace*{.05in}
\end{table}

\begin{table} 
\centering
\caption{Rank aggregator comparison for the Movielens dataset (partial rankings)}
\label{moviepartial}
\vspace*{.05in}
       \begin{tabular}{c@{\hspace*{.15in}}|c|c|c|c|c}
       	  $m$ 		& 20 	& 50 & 100 & 200 & 500 \\
	  \hline
      	Fas-Pivot  		&   328.8 	&  344.4 & 350.3 &  351.4 & 353.3  \\
	FasLP-Piovt     	&  328.6  & 344.4 & 350.3 & 351.4 & 353.5\\
   	InsertionComp  	& 386.3   & 390.2 & 392.6 & 393.1 & 393.0 \\
	Spearman      	&  332.9 & 347.3 & 352.5 & 353.5 &  355.4 \\ 
	LC-median       	&  334.2  & 350.4  & 355.4 & 355.9 & 359.1 \\
	LC-mode        	&  340.1  & 353.5 & 357.5 & 359.0 & 360.0 \\
      \end{tabular}
\vspace*{.05in}
\end{table}
%To sum up, Lehmer code-based algorithms can achieve a comparable ranking aggregation results with other widely used aggregation algorithms. Given Lehmer code-based algorithms is with low complexity ($O(mn)$ when $m>n$) and is easy for parallelization, Lehmer code-based algorithms will be extensively useful when handling large datasets.

\newpage

\renewcommand\refname

\newpage
\appendix
\section{Proof of Lemma 4.2} \label{ProofLemma1}
Before proceeding with the proof, we remark that some ideas in our derivatione have been motivated by Lemma 10.7 of~\cite{awasthi2014learning}.
%Without loss of generality, assume the centroid $\sigma_0^{-1}(i)=x_i$ for all $i\in[n]$, and hence the element $x_i$ is ranked $i$th in the centroid. 

Let $i\triangleq\sigma_0(u)$. Suppose that $n>j\geq i $ and that we want to prove statement 1) (the second case when $0< j\leq i$ may be handled similarly). When $i=1$, the underlying ratio is exactly equal to $\phi$. Hence, we only consider the case when $i>1$. Let $E=\{\sigma: \sigma(u)=j\}$ and $T=\{\sigma: \sigma(u)=j+1\}$. In this case, $\prob{[\sigma(u)=j]}=\prob{[E]}$ and $\prob{[\sigma(u)=j+1]}=\prob{[T]}$. Define the sets: 
\begin{align*}
E_{1}&=\{\sigma: \sigma(u)=j, \sigma_0(\sigma^{-1}(j+1))>i\}, \\
E_{2}&=\{\sigma: \sigma(u)=j, \sigma_0(\sigma^{-1}(j+1))<i\}, \\
T_{1}&=\{\sigma: \sigma(u)=j+1, \sigma_0(\sigma^{-1}(j))>i\}, \\
T_{2}&=\{\sigma: \sigma(u)=j+1, \sigma_0(\sigma^{-1}(j))<i\}. 
\end{align*}
Clearly, $\prob{[E]}=\prob{[E_{1}]}+\prob{[E_{2}]}$ and $\prob{[T]}=\prob{[T_{1}]}+\prob{[T_{2}]}$. By swapping $u$ and $\sigma^{-1}(j+1)$, we can construct two bijections $E_{1}\leftrightarrow T_{1}$ and $E_{2}\leftrightarrow T_{2}$. Statement $1)$ can then be easily proved by using the following three claims: 
\begin{align}
 \prob{[T_{1}]}&=\phi \prob{[E_{1}]}, \nonumber\\
\prob{[T_{2}]}&=\frac{1}{\phi} \prob{[S_{2}]}, \nonumber \\
0<\prob{[T_{2}]}&\overset{a)}{\leq} \phi_{1:n-1} \prob{[T_{1}]}.  \label{Lemma1ProbRatio}
\end{align}
Observe that inequality is achieved in a) when $j=n-1$. The first two claims are straightforward to check, and hence we only prove the third claim. 

Consider a mapping from $T_{2}$ to $T_{1}$ based on circular swapping of elements, and let $\sigma\in T_{2}$. Since $\sigma(u)-1=j\geq i$ and $\sigma_0(\sigma^{-1}(j))<i$, there must exist an element $x$ such that $\sigma_0(x)>\sigma_0(u)$ and $\sigma(x)<j$. Choose the element $x$ with the largest corresponding value of $\sigma(x)$ and construct a new ranking $\sigma'$ such that
\begin{equation*}
\sigma'(y)=\left\{
\begin{array} {lc}
\sigma (y),\quad &  \text{ if } \sigma(y)<\sigma(x)\;\text{or}\; \sigma(y)\geq\sigma(u),   \\
 \sigma(y) -1, \quad & \text{ if } \sigma(x)< \sigma(y)\leq \sigma(u),  \\
j, \quad &  \text{ if } \sigma(y)=\sigma(x).\\
\end{array}
\right.
\end{equation*}
It is easy to see that $\sigma'\in T_{1}$. Given that all elements ranked between $x$ and $u$ in $\sigma$ have rank higher than $\sigma_0(x)$, we have $\prob{[\sigma]}=\phi^{\sigma(u)-\sigma(x)-1}\prob{[\sigma']}=\phi^{j-\sigma(x)}\prob{[\sigma']}$. Note that the above mapping is neither a bijection nor an injection. Denote the mapping by $\mathcal{M}: T_{j,2}\rightarrow T_{2}$. For each $\sigma'\in T_{1}$, define $T_{2,\sigma'}\subset T_{2}$, so that for all $\sigma\in T_{2,\sigma'}$, $\mathcal{M}(\sigma)=\sigma'$. Then, $\cup_{\sigma'\in T_{1}}T_{2,\sigma'}=T_{2}$ forms a partition of the set $T_{2}$. Next, consider two distinct rankings $\sigma_1,\sigma_2\in T_{2,\sigma'}.$ These rankings must rank the element $x$ differently, i.e., one must have $\sigma_1(x)\neq \sigma_2(x)$. Therefore, $\prob{[T_{2,\pi'}]}\leq \prob{[\pi']} \phi_{1:j-1}=\prob{[\pi']}\phi_{1:j-1}$. As a result, $\prob{[T_{2}]}\leq\prob{[T_{1}]}\phi_{1:n-1}$, which proves the third claim. We conclude by observing that the condition under which equality is achieved in the bound stated in the lemma is exactly the same condition under which equality is achieved in the bound stated in the third claim. 
%Note that a similar Lemma was stated in \cite{awasthi2014learning}. However, the result was incorrect because the constructive bijective mapping is not ``bijective''.  

\section{Proof of Lemma 4.3} \label{ProofLemma2}

Let $i\triangleq\sigma_{0,A}(u)$. Suppose that $n>j\geq i$ and that we want to prove statement 1) (the case when $0< j\leq i$ may be handled similarly). Let $E=\{\pi: \pi_{A}(u)=j\}$ and $T=\{\pi: \pi_{A}(u)=j+1\}$. The left-hand-side in the statement of 1) equals the ratio $\frac{\prob{[T]}}{\prob{[E]}}$. %The proof of Lemma 5.2 follows a similar methodology of Lemma~5.1: we partition all the rankings into different non-overlapping classes $T=\cup_i T^{(i)}$ and $E=\cup_i E^{(i)}$ and try to prove the upper bound of $\frac{\prob{[T^{(i)}]}}{\prob{[E^{(i)}]}}$ for each class $i$.
Note that removing a fixed number of elements in $S$ of lowest (or highest) rank in the centroid ranking does not change the probability of the ranking involving the remaining elements (see Lemma~\ref{extLemma1} for the proof). We can hence assume that $\sigma_{0,A}^{-1}(1)$ is the element with highest rank in $\sigma_0$. %and $\sigma_{0,A}^{-1}(|A|)$ is the element with highest ranking in $\sigma_0$.

When $i=1$, for any ranking $\sigma$ in $T$, we can swap the element $u$ with the element $x\in A$ for which $\sigma_A(x)=\sigma_A(u)-1$ to obtain another ranking $\sigma'\in E$. Moreover, it is easy to check that $\prob{[\sigma']}\phi\geq \prob{[\sigma]}$, so that the ratio in the statement 1) does not exceed $\phi$. Note that we have inequality $``\geq''$ instead of equality $``=''$ in $\prob{[\sigma']}\phi\geq \prob{[\sigma]}$, since there may potentially exists other elements in $S/A$ ranked between $x$ and $u$ in $\sigma$. 

Next, consider the case when $i>1$. Define the sets 
\begin{align*}
E_{1}&=\{\sigma: \sigma_A(u)=j, \sigma_{0,A}(\sigma_A^{-1}(j+1))>i\}, \\
E_{2}&=\{\sigma: \sigma_A(u)=j, \sigma_{0,A}(\sigma_A^{-1}(j+1))<i\}, \\
T_{1}&=\{\sigma: \sigma_A(u)=j+1, \sigma_{0,A}(\sigma_A^{-1}(j+1))>i\}, \\
T_{2}&=\{\sigma: \sigma_A(u)=j+1,\sigma_{0,A}(\sigma_A^{-1}(j+1))<i\}. 
\end{align*}
Then, $\prob{[E]}=\prob{[E_{1}]}+\prob{[E_{2}]}$ and $\prob{[T]}=\prob{[T_{1}]}+\prob{[T_{2}]}$. By swapping $u$ and $\sigma_A^{-1}(j+1)$, we can construct two bijections $E_{1}\leftrightarrow T_{1}$ and $E_{2}\leftrightarrow T_{2}$ as follows. 

Let us consider a finer partition of $T_{2}$ in terms of permutations with four labels. More precisely, associate each ranking $\sigma\in T_{2}$ with a label vector $(x_1,x_2,\ell_1,\ell_2)$, where: 
\begin{description}
\item{$x_1$} =  $\pi^{-1}(j)$. Note that $\sigma_{0,A}(x_1)< i$ due to the definition of $T_2$. % the element ranked $j+1$th in $\pi_A$, i.e., $\pi_A(j+1)$. Hence, $k_1<i$.
\item{$x_2$} =  $\arg\max_{x: \sigma_{0,A}(x)>i,  \sigma_A(x)< \sigma_A(u)} \sigma_A(x)$; the label $x_2$ is well-defined due to the pigeon-hole principle.
%the element such that $k_2>i$ and ranked with largest $\pi_A^{-1}(k_2)$ while $\pi_A^{-1}(k_2)<\pi_A^{-1}(i)$. ($k_2$ exists because of pigeon-hole principle.)
%\item{$\boldsymbol{\ell}$}: $\boldsymbol{\ell}$ is a $4\times1$ vector and can be written as $(\ell_1,\ell_2,\ell_3,\ell_4)$. $\ell_{k}$ corresponds to the cardinality of a set $F_k$, where $F_k$ is defined as following 
%\begin{equation} \nonumber
%F_{k}=\left\{\begin{array}{lc} \{x: \sigma_{0}(\sigma_A^{-1}(j-1))<\sigma_{0}(x)<\sigma_{0}(u), \sigma(\sigma_A^{-1}(j-1))<\sigma(x)<\sigma(\sigma_A^{-1}(j))\} & k=1 \\ 
%\{x: \sigma_{0}(u)<\sigma_{0}(x)<\sigma_{0}(\sigma_A^{-1}(j)), \sigma(\sigma_A^{-1}(j-1))<\sigma(x)<\sigma(\sigma_A^{-1}(j))\} & k=2 \\
%\{x: \sigma_{0}(\sigma_A^{-1}(j-1))<\sigma_{0}(x)<\sigma_{0}(u), \sigma(\sigma_A^{-1}(j))<\sigma(x)<\sigma(u)\} & k=3 \\
%\{x:\sigma_{0}(u)<\sigma_{0}(x)<\sigma_{0}(\pi_A^{-1}(j)), \sigma(\sigma_A^{-1}(j))<\sigma(x)<\sigma(\sigma_A^{-1}(j+1))\} & k=4 
%\end{array}\right.
%\end{equation}

\item{$\ell_1$} = the cardinality of the set $F_1$ defined as 
\begin{align*}
F_1=\{x\in[n]: \sigma_{0}(x_1)<\sigma_{0}(x)<\sigma_{0}(u), \sigma(\sigma_A^{-1}(j-1))<\sigma(x)<\sigma(x_1)\}. 
\end{align*}
\item{$\ell_2$} =  the cardinality of the set $F_2$ defined as 
\begin{align*}
F_2=\{x\in[n]: \sigma_{0}(x_1)<\sigma_{0}(x)<\sigma_{0}(u), \sigma(x_1)<\sigma(x)<\sigma(u)\} .
\end{align*}
\end{description}
We summarize those labels in a vector $L=(x_1,x_2,\ell_1,\ell_2)$ and thus partition $T_2$ according to different label vectors $L$, i.e., 
\begin{align} \label{part1}
T_{2}=\cup_{L} T_{2,L}. 
\end{align}
A ranking in $T_{2}$ is in $T_{2,L}$ if its corresponding label vector equals $L$.

We further construct a mapping $\mathcal{M}$ from $T_2$ to $T_1$ by swapping elements ranked between $x_1$ and $x_2$, so that $\sigma'=\mathcal{M}(\sigma)$ equals
\begin{equation*}
\sigma_A'(x)=\left\{
\begin{array} {lc}
j,\quad & \sigma_A(x)=\sigma_A(x_2),  \\
 \sigma_A(x)-1, \quad &  \sigma_A(x_2)<\sigma_A(x)< j, \\
\sigma_A(x)-1, \quad & \sigma_A(x)=\sigma_A(x_1),  \\
 \sigma_A(x), & \text{for other $x\in A$}.
\end{array}
\right.
\end{equation*}
The above mapping basically performs circular swapping by moving $x_2$ to the position one rank higher and adjacent to $u$ and by moving each element in $A$ between $x_2$ and $x_1$, including $x_1$, to a higher position adjacent to the original one. Based on $\mathcal{M}$, one can also form a partition of $T_{1}$ as
\begin{align} \label{part2}
T_{1} =(\cup_{L} T_{1,L})\cup T_{1,L^c}
\end{align}
where $T_{1,L}$ contains the rankings mapped from $T_{2,L}$ via $\mathcal{M}$. 
Note that $T_{1,L^c}$ denote the ``remainder set'' of permutations that do not have a preimage in $T_{2}$. In this remainder set, a ranking $\sigma$ has the property that the elements $\sigma_A^{-1}(j)$ and $\sigma_A^{-1}(j-1)$ are both ranked lower than $u$ in the centroid ranking. Since the swapping operations establish a bijection between $E_{1}\leftrightarrow T_{1}$ and $E_{2}\leftrightarrow T_{2}$, one can also partition $E_{1},E_{2}$ as
\begin{align} \label{part3}
E_{1}&=(\cup_{L} E_{1,L})\cup E_{1,L^c}, \\
E_{2}&=\cup_{L} E_{2,L}. \label{part4}
\end{align}

Let $R(\mathcal{L})$ denote $\frac{\prob{[\cup_{L\in\mathcal{L}}(T_{1,L}\cup T_{2,L})]}}{\prob{[\cup_{L\in\mathcal{L}}(E_{1,L}\cup E_{2,L})]}}$ and let $R(L)$ denote the same type of ratio but for a specific choice of $L$, i.e., $\frac{\prob{[T_{1,L}\cup T_{2,L}]}}{\prob{[E_{1,L}\cup E_{2,L}]}}$. Also, let $\mathcal{L}_0$ denote the set of all possible values of $L$. To prove the upper bound on $\frac{\prob{[T]}}{\prob{[E]}}$, we proceed through four steps. 
\begin{enumerate}
\item Partition $T$ and $E$ and verify the validity of \eqref{part1}, \eqref{part2}, \eqref{part3} and \eqref{part4}.
\item Prove that $\frac{\prob{[T_{1,L^c}]}}{\prob{[E_{1,L^c}]}}\leq \phi$.
\item Prove the upper bound for $R(L)$ when $\ell_1=\ell_2$.
\item Prove the upper bound for $R(L\cup L'),$ where  $L=(k_1,\,k_2,\ell_1,\ell_2)$ and $L'=(k_1,\,k_2,\ell_2,\ell_1),$ for the case that $\ell_1\neq \ell_2$. 
\end{enumerate}
%$R(\mathcal{L}_0)\leq R(\mathcal{L}_1)$ where $\mathcal{L}_1=\{L=(k_1,k_2,\boldsymbol{\ell}): \ell_{2}=\ell_{4}=0\}$.
%\item $R(\mathcal{L}_1)\leq R(\mathcal{L}_2)$ where $\mathcal{L}_2=\{L=(k_1,k_2,\boldsymbol{\ell}): \ell_{2}=\ell_{4}=0, \ell_{1}=\ell_{3}\}$.
%\item $R(\mathcal{L}_2)\leq \max_{l\in\mathbb{N}}\frac{\phi+\phi^{l}\phi_{2:n-l-1}}{1+\phi^{2l}\phi_{3:n-l}}$, where RHS$=\frac{\phi_{1:n-1}}{1+\phi_{3:n}}$ when $\phi+\phi^2<1+\phi^n$.
The second step is easy to prove by directly swapping $\sigma_A^{-1}(j)$ and $u$ in any given ranking $\sigma\in T_{1,L^c} $. We hence only need to establish the validity of the results in Steps 3 and 4. 
%The following lemma is obvious correct. 
%\begin{lemma} \label{LemmaLemma1}
%If $a_1(x),a_2(x),...,a_k(x)$ and $b_1(x),b_2(x),...,b_k(x)$ are two types of positive functions depending on real number $x$. If $\frac{a_{i}(x)}{b_{i}(x)}$ monotonously increases (decreases) in terms of $x$ for all $i$, then $\frac{\sum_{i} a_{i}(x)}{\sum_{i} b_{i}(x)}$ also monotonously increases (decreases resp.) in terms of $x$.
%\end{lemma}
%
%With this lemma, to prove the fact 2), we can show that for all $L=(k_1,k_2,\boldsymbol{\ell})\in \mathcal{L}_0$, the upper bound of ratio $R(L)$ will not decrease if we set $\ell_{2}=\ell_{4}=0$. and we will see that setting $\ell_{2}=\ell_{4}=0$ will not decrease the upper bound of $R(L)$. 

For any $L=(k_1,k_2,\ell_1,\ell_2)$, the following claims hold:
\begin{align}
&\prob{[E_{1,L}]}\geq\phi^{-1} \prob{[T_{1,L}]}, \nonumber\\
&\prob{[E_{2,L}]}=\phi^{1+2\ell_{2}} \prob{[T_{2,L}]},\nonumber\\
&\prob{[T_{2,L}]}\leq \phi^{2\ell_{1}} f_L\prob{[T_{1,L}]},\;\text{where}\; f_L\leq \phi_{1:|A|-2-\ell_{1}-\ell_{2}}, \label{ProbRatio}
\end{align}
where the first two claims are easy to prove, while the equation~\eqref{ProbRatio} may be verified similarly as~\eqref{Lemma1ProbRatio} in the proof of Lemma 5.2 (See Appendix~\ref{ProofLemma1}). 

For any $\sigma\in T_{2,L}$, $\sigma'=\mathcal{M}(\sigma)\in T_{1,L}$. Given that all the elements in $A$ ranked between $x_2$ and $u$ in $\sigma$ are ranked lower than $u,\,x_2$ in the centroid, and due to swapping, we have 
\begin{align*}
\frac{\prob{[\sigma]}}{\prob{[\sigma']}}\leq\phi^{\sigma_A(u)-\sigma_A(x_2)-1+2\ell_{1}}=\phi^{j-\sigma_A(x_2)+2\ell_{1}}.
\end{align*}
Consider two distinct rankings $\sigma_1,\sigma_2\in T_{2,L}.$ If $\mathcal{M}(\sigma_1)=\mathcal{M}(\sigma_2)$, both rankings rank the element $x_2$ differently over $A$, i.e., $\sigma_{1,A}(x_2)\neq \sigma_{2,A}(x_2)$. Therefore, we must have
 \begin{align*}
 \sum_{\sigma:\mathcal{M}(\sigma)=\sigma'}\frac{\prob{[\sigma]}}{\prob{[\sigma']}}\leq \phi^{2(\ell_{1}+\ell_{2})}\phi_{1:|A|-2-\ell_{1}-\ell_{2}}.
\end{align*} 
By examining all mappings from $\prob{[T_{2,L}]}$ to $\prob{[T_{1,L}]}$, we conclude that $f_L(\phi)=\phi^{-2(\ell_{1}+\ell_{2})}\frac{\prob{[T_{2,L}]}}{\prob{[T_{1,L}]}}\leq \phi_{1:|A|-2-s}$, which establishes the third claim in \eqref{ProbRatio}. 
 
Substituting the above expressions into $R(L)$, we have 
\begin{align}\label{keyineq1}
R(L) \leq \frac{1+\phi^{2\ell_{1}}\phi_{1:|A|-2-\ell_{1}-\ell_{2}}}{\phi^{-1}+\phi^{1+2(\ell_{1}+\ell_{2})}\phi_{1:|A|-2-\ell_{1}-\ell_{2}}}.
\end{align}
%It is easy to observe the correctness following result because of the monotonicity. 
%\begin{lemma} \label{Lemma2Lemma1}
%Given $L=(k_1,\,k_2,\,(\ell_1,\ell_2,\ell_3,\ell_4))$. Let $\tilde{L}=(k_1,\,k_2,\,(\ell_1,0,\ell_3,0))$, then we have
%\begin{align*}
%R(L)\leq R(\tilde{L}).
%\end{align*}
%\end{lemma}
%Since the RHS of \eqref{keyineq1} will not decrease when $\ell_2=\ell_4=0$ and 
Suppose next that $\ell_{1}=\ell_3=\ell$. Then,
\begin{align}\label{keyineq1}
R(L) \leq \frac{1+\phi^{2\ell}\phi_{1:|A|-2-2\ell}}{\phi^{-1}+\phi^{1+4\ell}\phi_{1:|A|-2-2\ell}}.
\end{align}
which completes the proof of the bound in Step 3.

Let us now consider the bound in Step 4. When $\ell_{1}\neq\ell_{2}$, direct optimization over $\ell_1, \ell_2$ cannot yield the 
required upper bound as $\ell_{2}\rightarrow \infty$ may increase the right-hand-side of \eqref{keyineq1}. Hence, in addition 
to $L=(k_1,\,k_2,\,\ell_1,\,\ell_2)$, let us also simultaneously consider $L'=(k_1\,,k_2\,,\ell_2,\,\ell_1)$, 
as a larger $\ell_{2}$ will yield a smaller $R(L')$, . 

Without loss of generality, suppose that $\ell_2>\ell_1$. First, we prove the following Lemma.
\begin{lemma} \label{Lemma2Lemma2}
For a pair $(L,\,L')$ defined as above with $\ell_2>\ell_1$, one has
\begin{align*}
\frac{\prob{[T_{1,L}]}}{\prob{[T_{1,L'}]}}\leq \phi^{(\ell_{2}-\ell_{1})}.
\end{align*}
\end{lemma}
\begin{proof}
Recall the definition of the set $F_2$ and the fact that for a ranking $\sigma \in T_{1,L}$, $\sigma$ is obtained via $\mathcal{M}(\pi)$ for some $\pi\in T_{2,L}$. Hence, in $\sigma$, the elements in $F_1$ are now ranked higher than $x_2$ and lower than $x_1$, while the elements in $F_2$ are now ranked higher than $u$ and lower than $x_2$. Based on this structure of $\sigma$, for each ranking $\sigma \in T_{1,L}$, one may perform a swapping operation to obtain another ranking $\sigma'$ in $T_{1,L'}$. The swapping constitutes a bijection. To see this, consider the set of $\ell_2-\ell_1$ elements with highest rank in $\sigma_{F_2}$ (Note that we assumed $\ell_2>\ell_1$ but could have otherwise considered 
the $\ell_1-\ell_2$ elements with lowest rank in $\sigma_{F_1}$.). Swapping the element $x_2$ and the selected elements in $F_2$ ranked 
from high to low yields a ranking $\sigma'\in T_{1,L'}$. Since for any element $x\in F_2$, $\sigma_{0}(x)<\sigma_{0}(u)<\sigma_{0}(x_2)$, we have $\prob([\sigma])\leq\prob([\sigma'])\phi^{\ell_{2}-\ell_{1}}$. This completes the proof.
\end{proof}
Let $P\triangleq\frac{\prob{[T_{1,L}]}}{\prob{[T_{1,L'}]}}$. Substituting the results of all claims~\eqref{ProbRatio} into $R(L\cup L')$, we obtain
\begin{align}\label{keyineq1}
R(L\cup L') &= \frac{\prob{[T_{1,L}\cup T_{2,L}]}+\prob{[T_{1,L'}\cup T_{2,L'}]}}{\prob{[E_{1,L}\cup E_{2,L}]}+\prob{[E_{1,L'}\cup E_{2,L'}]}}\\
&\leq\frac{P(1+\phi^{2\ell_1}\phi_{1:|A|-2-\ell_1-\ell_2})+(1+\phi^{2\ell_2}\phi_{1:|A|-2-\ell_1-\ell_2})}{(P+1)(\phi^{-1}+\phi^{1+2(\ell_{1}+\ell_{2})}\phi_{1:|A|-2-\ell_{1}-\ell_{2}})} \\
 &\overset{b)}{\leq}  \frac{1+\phi^{\ell_{1}+\ell_{2}}\phi_{1:|A|-2-\ell_{1}-\ell_{2}}}{(\phi^{-1}+\phi^{2(\ell_{1}+\ell_{2})} \phi_{2:|A|-1-\ell_{1}-\ell_{2}})}.
\end{align}
Here, the inequality b) follows from Lemma~\ref{Lemma2Lemma2}.  By using $|A|\leq n$ and letting $\ell=\ell_{1}+\ell_{2}$, we obtain  
\begin{align*}
R(L\cup L')\leq \frac{1+\phi^{\ell}\phi_{1:n-\ell-2}}{\phi^{-1}+\phi^{2\ell}\phi_{2:n-\ell-1}},
\end{align*}
which completes the proof of the result in Step 4. 
%By letting $n\rightarrow \infty$,  we have RHS is monotonically decreasing with $l$. Therefore, $\frac{\phi+\phi^{l}\phi_{2:n-l-1}}{1+\phi^{2l}\phi_{3:n-l}}< \frac{\phi}{1-\phi+\phi^3}$. 

%For a fixed $n$, and assuming that $\phi+\phi^2< 1+\phi^n$, we can prove that the maximum achieve the maximum with $\ell=0$ in the following lemma. Therefore, $\frac{\phi+\phi^{\ell}\phi_{2:n-\ell-1}}{1+\phi^{2\ell}\phi_{3:n-\ell}}\leq \frac{\phi_{1:n-1}}{1+\phi_{3:n}}.$
\begin{lemma} \label{Lemma2Lemma3}
If $\phi+\phi^2< 1+\phi^n$, for all $\ell\in\mathbb{N}$, one has
\begin{align*}
\frac{1+\phi^{\ell}\phi_{1:n-\ell-2}}{\phi^{-1}+\phi^{2\ell}\phi_{2:n-\ell-1}}\leq \frac{\phi_{1:n-1}}{1+\phi_{3:n}}.
\end{align*}
\end{lemma}
\begin{proof}
First, for $\ell\geq 1$,
\begin{align*} 
\frac{\phi^{\ell+1}}{\phi_{2\ell+1:2\ell+2} -\phi^{n+\ell}}\geq \frac{1}{\phi+\phi^2-\phi^{n-1}}> 1. 
\end{align*}
Thus,
\begin{align*} 
\frac{\phi_{2:\ell+1}}{\phi_{3:2\ell+2}-\phi_{n+1:n+\ell}} = \frac{\sum_{t=1}^{\ell} \phi^{t+1}}{\sum_{t=1}^{\ell} \phi_{2t+1:2t+2}- \sum_{t=1}^{\ell}\phi^{n+t} }> 1. 
\end{align*}
Since we also have $\frac{\phi_{1:n-1}}{1+\phi_{3:n}}<1$, it follows that 
\begin{align*}
\frac{\phi_{1:n-1}}{1+\phi_{3:n}}>\frac{\phi_{1:n-1}-\phi_{2:\ell+1}}{1+\phi_{3:n}-(\phi_{3:2\ell+2}-\phi_{n+1:n+\ell})}=\frac{\phi+\phi^{\ell}\phi_{2:n-\ell-1}}{1+\phi^{2\ell}\phi_{3:n-\ell}}.
\end{align*}
This proves the claimed result and completes the proof of the lemma.
\end{proof}

\section{Proof of Lemma 4.5}\label{ProofLemma3}
Lemma 4.5 is a corollary of the following lemma. 
\begin{lemma} \label{geneLemma3}
Let $\sigma\sim\;$ MM$(\sigma_0,\phi)$. If two subsets of elements $A$, $A'$ satisfy $A' = A \cup \{{x\}}$, where $x \notin A$, and if $u\in A$, then for all $t\in[|A|]$ one has
\begin{description}
\item{1)} $\prob{[\sigma_{A}(u)\geq t]} \leq \prob{[\sigma_{A'}(u)\geq t]}$.
\item{2)} $\prob{[\sigma_{A}(u)\leq t]} \leq \prob{[\sigma_{A'}(u)\leq t+1]}$.
\end{description}
\end{lemma}
\begin{proof}
Because of symmetry, it suffices to prove the first claim only. The left-hand-side of the first inequality equals the probability of the event that the element $u$ is ranked in the $t$-th position or lower within the set of elements in $A$. The right-hand-side of the inequality equals the probability of the event that the element $u$ is ranked in the $t$-th position or lower within the set of elements in $A'$. Since $A'$ is the union of $A$ and another element $x \notin A$, inserting $x$ into a ranking may only increase the rank of already present elements.
%, so the event of RHS includes the event of LHS. 
\end{proof}
Now, consider Lemma 4.5 in the main text. Choose an element $x\in S$ if there is such and element that satisfies 
$\sigma_0(x)>\sigma(u)$. Let $A'=A\cup\{x\}$. Then, the statement of the above result implies that the probability that element $u$ is ranked lower than or equal to its correct rank $\sigma_{0,A}(u)$ will increase if we add an element $x$ to $A$ that is ranked lower than $u$ in $\sigma_0$. Therefore, by removing all elements from $A$ that are ranked lower than $u$ in $\sigma_0$, we obtain a new subset $A''$ and consequently have $\prob{[\sigma_{A}(u)\geq \sigma_{0,A}(u)]}\geq \prob{[\sigma_{A''}(u)\geq \sigma_{0,A''}(u)]}$. Note that $u$ is the element with the lowest rank in $\sigma_{0,A''}$. Therefore, it is easy to check that $\prob{[\sigma_{A''}(u)=\sigma_{0,A''}(u)]}\geq\frac{1}{\phi_{0:|A''|-1}}$. Then, one has the inequality $\prob{[\sigma_{A}(u)\geq \sigma_{0,A}(u)]}\geq  \frac{1}{\phi_{0:|A''|-1}}.$ %indicates statement b) in Lemma 5.5. Similarly, we can prove the statement 1) of Lemma 5.5.   

\section{Proof of Lemma 4.7} \label{ProofLemma6}
For convenience, we restate Lemma 4.7 first. 
\begin{lemma} \label{Lemma6}
Let $\sigma\sim\;$GMM$(\sigma_0,\phi)$ and let $A$ be a subset of elements containing an element $u$. Let $A'=A-\{x\in A: x\neq u, \sigma_{0,A}(x)\leq\sigma_{0,A}(u) \}$. Define $W=\{r_{\sigma_{A}(u)}\leq r_{\sigma_{0,A}(u)}\}$, $Q_j=\{r_{\sigma_{A}(u)}=j+r_{\sigma_{0,A}(u)},\;l_{\sigma_A(u)}\leq r_{\sigma_{0,A}(u)}\}$, $W'=\{r_{\sigma_{A'}(u)}\leq r_{\sigma_{0,A'}(u)}\}$ and $Q_j'=\{r_{\sigma_{A'}(u)}=j+r_{\sigma_{0,A'}(u)},\;l_{\sigma_A'(u)}\leq r_{\sigma_{0,A'}(u)}\}$. Then, the following two claims hold. 
\begin{align}
\prob{[W]} + &\sum_{j=1}^{|A|-r_{\sigma_{0,A}}(u)}V_j\prob{[Q_j]}\nonumber\\
\geq& \prob{[W']} + \sum_{j=1}^{|A'|-r_{\sigma_{0,A'}}(u)}\frac{1}{j+1}V_j\prob{[Q_j']} \label{partialmedian4}\\
\geq& 1-\frac{1}{2}\phi^{1/2}-\frac{1}{2}\phi. \label{partialmedian5}
\end{align} 
\end{lemma}
\begin{proof}
The idea behind the proof is similar to that of the proof of Lemma~\ref{geneLemma3}. 
Our first goal is to show that removing the element $x$ from $A$ that is ranked highest in $\sigma_{0,A}$ decreases the left-hand-side of~\eqref{partialmedian4}. Then, by induction, we may prove~\eqref{partialmedian4}. 

For simplicity, let $A''=A-\{x\}$. Note that because of the choice of the rank of the element $x$ in $\sigma_{0,A}$, we have $r_{\sigma_{0,A''}(u)}=r_{\sigma_{0,A}(u)}-1$ and $l_{\sigma_{0,A''}(u)}=l_{\sigma_{0,A}(u)}-1$. For a ranking $\sigma\in \{\sigma: l_{\sigma_{A}}(u)\leq r_{\sigma_{0,A}}(u)\}$, the removal of $x$ produces another ranking $\sigma''$. When $r_{\sigma_A(x)}< r_{\sigma_{A}(u)}$, $\sigma$ and $\sigma''$ will contribute the same ``vote'' to the left-hand-side of~\eqref{partialmedian4}. When $r_{\sigma_A(x)}= r_{\sigma_{A}(u)} $, $\sigma''$ contributes the same vote when $\sigma\in W$, or smaller vote when $\sigma\in Q_j$ for some $j$.  When $r_{\sigma_{A}(x)}> r_{\sigma_{A}(u)}$, $\sigma''$ will always contribute a smaller vote. For a ranking $\sigma\notin \{\sigma: l_{\sigma_{A}}(u) \leq r_{\sigma_{0,A}}(u)\}$, both $\sigma$ and $\sigma''$ contribute a zero vote. Therefore, removing $x$ from $A$ strictly decreases the left-hand-side of~\eqref{partialmedian4}.

We now prove inequality \eqref{partialmedian5}. Note that due to the definition of $A'$, we have $r_{\sigma_{0,A'}}(y)=1$.  %Since removing all the elements in $S$ with first several buckets of the centroid ranking does not change the generated probability, i.e., $\prob{[\sigma]}=\prob{[\sigma_{S/\{\mathcal{B}_{\sigma_1},...,\mathcal{B}_{\sigma_k}\}}]}$, 
Also, due to Lemma~\ref{extLemma1} of this document, we can further assume that $r_{\sigma_0}(y)=1$, i.e., that $y$ is the only element in the first bucket of $\sigma_0$. Because of its definition, $W'$ includes the rankings $\sigma$ such that $y$ is the only element in the first bucket of $\sigma_A'$ while $Q_j'$ includes the rankings $\sigma$ such that there are, in addition to $y$, some $j$ other elements in the first bucket of $\sigma_A'$. 

Partition $W'$ according to the size of the second bucket of $\sigma_A'$, i.e., let 
$W'= \cup_{|\mathcal{B}_2(\sigma_A')|=j} W_{j}'$. Then, we can construct a bijection from $W_{j}'$ to $Q_{j}'$ 
by putting $y$ into the second bucket. It is easy to check that 
\begin{align*}
\prob{[Q_j']}\leq \prob{[W_j']}\phi^{j/2}.
\end{align*}
Denote the set of partial rankings $\sigma$ for which the first bucket of $\sigma_A$ does not contain $y$ but some $j$ other elements in $U_j'$. 
We can also construct a mapping from $Q_{j}'$ to $U_j'$ by moving $y$ from the first bucket to any other possible position higher than the first bucket. Hence, we have 
\begin{align*}
\prob{[U_j']}\leq \prob{[Q_j']}\phi^{j/2}(1+\phi^{1/2}+\phi+\cdots)\leq\prob{[Q_j']}\frac{\phi^{j/2}}{1-\phi^{1/2}}.
\end{align*}
Let $Z_j'=W_j'\cup Q_j'\cup U_j'$, so that $\cup_{j}Z_j'$ covers all possible partial rankings. Hence,
\begin{align*}
 \frac{\prob(W_j')+\frac{1}{j+1}\prob(Q_j')}{\prob(Z_j')}=\frac{\prob(W_j')+\frac{1}{j+1}\prob(Q_j')}{\prob(W_j')+\prob(Q_j')+\prob(U_j')}\geq  \frac{1+\frac{1}{j+1}\phi^{j/2}}{1+\phi^{j/2}+\frac{\phi^j}{1-\phi^{1/2}}} \geq \frac{1+\frac{1}{2}\phi^{1/2}}{1+\phi^{1/2}+\frac{\phi}{1-\phi^{1/2}}}=1-\frac{1}{2}\phi^{1/2}-\frac{1}{2}\phi,
\end{align*}
where the second inequality follows from $\phi^{1/2}+\phi<1$. This completes the proof. 
\end{proof}

\section{Proof of the Main Results} \label{ProofThm1}
\subsection{Proof for permutation aggregation}
Let the Lehmer code of the output permutation $\sigma$ be denoted by $\hat{\bfc}_{\sigma}$. We say that the LCA algorithm succeeds 
if $\sigma=\sigma_0$, or equivalently, if $\hat{\bfc}_{\sigma}= \bfc_{\sigma_0}$. Given that $\hat{\bfc}_{\sigma}(1)=0=\bfc_{\sigma_0}(1)$, by using the union bound, we arrive at
\begin{align*}
\prob{[\sigma=\sigma_0]}=\prob{[\hat{\bfc}_{\sigma}= \bfc_{\sigma_0}]}\geq 1-\sum_{t=2}^{n-1} \prob{[\hat{\bfc}_{\sigma}(t)\neq \bfc_{\sigma_0}(t)]}.
\end{align*}
In Section 4, we explained that the algorithm based on the Lehmer code $\hat{\bfc}_{\sigma}$ may be viewed as a form of InsertionSort, in which during the $t$-th iteration one places the element $t$ at the $(t-\hat{\bfc}_{\sigma}(t))$-th position over the subset of elements $S_t=[t]$. With this specific choice of subset $S_t$, for any permutation $\pi$, we have $\pi_{S_t}(t)=t-\bfc_{\pi}(t)$. Hence, the event $\{\hat{\bfc}_{\sigma}(t)\neq \bfc_{\sigma_0}(t)\}$ is equivalent to the event $\{\sigma_{S_t}(t)\neq\sigma_{0,S_t}(t)\}$, which we denote by $D_t$. For convenience, we let $i\triangleq \sigma_{0,S_t}(t)$.

Given that the ranking $\sigma_k\in \Sigma$ is sampled from a MM$(\sigma_0,\phi)$ distribution, we also define a random variable $X_k(j)=1_{\{\sigma_{k,S_t}(t)=j\}}$ to indicate whether the element $t$ is ranked at the $j$-th position in $\sigma_{k,S_t}$. Therefore, $\sum_{k\in[m]}X_k(j)$ equals the number of rankings in $\Sigma$ in which element $t$ is ranked at the $j$-th position.

\subsubsection{Proof of Theorem 4.1 (Mode)}
Given that we aggregate using the mode function, 
we have $\sigma_{S_t}(t)=\arg\max_{j\in [t]}\sum_{k\in[m]}X_k(j)$. In what follows, we aim to prove an upper bound on $\prob{[\sigma_{S_t}(t)\neq\sigma_{0,S_t}(t)]}=\prob{[D(t)]}$. 

To this end, let $q= \frac{\phi_{1:n-1}}{1+\phi_{3:n}}$, so that when $\phi+\phi^2<1+\phi^n$, we have $q<1$. Based on Lemma~4.3 in the main text, we have 
\begin{align*}
\prob{[X_k(i)=1]}=\frac{\prob{[X_k(i)=1]}}{\sum_{j=1}^t\prob{[X_k(j)=1]}}=\frac{\prob{[\sigma_{k,S_t}(t)=i]}}{\sum_{j=1}^t\prob{[\sigma_{k,S_t}(t)=j]}}\geq \frac{1}{1+2q/(1-q)}=\frac{1-q}{1+q}.
\end{align*}
Moreover, if $\mathcal{E}=\mathbb{E}{[X_k(i)-X_k(j)]}$, then 
\begin{align*}
\mathcal{E} \geq \prob{[X_k(i)=1]}(1-q^{|i-j|})\geq \frac{(1-q)^2}{1+q}. 
\end{align*} 
Therefore, since the $\sigma_k, \, k \in [m],$ are i.i.d, Hoeffding's inequality establishes 
\begin{align*}
\prob{[\sum_{k\in[m]}X_k(i)<\sum_{k\in[m]}X_k(j)]}=\prob{[\sum_{k\in[m]}(X_k(i)-X_k(j))\leq 0]}\leq \exp(-\frac{m\mathcal{E}^2}{2})\leq \exp(-\frac{m(1-q)^4}{2(1+q)^2}). 
\end{align*} 
Hence, 
\begin{align*}
\prob{[D_t]}\leq \sum_{j\in[t],j\neq i} \prob{\left[\sum_{k\in[m]}X_k(i)<\sum_{k\in[m]}X_k(j)\right]} < (t-1)\exp(-\frac{m(1-q)^4}{2(1+q)^2}).
\end{align*}
As a result, for $m\geq c\log\frac{n^2}{2\delta}$ with $c=\frac{2(1+q)^2}{(1-q)^4}$ and $q=\frac{\phi_{1:n-1}}{1+\phi_{3:n}}$, we have $\prob{[\sigma=\sigma_0]}> 1-\delta$. 

\subsubsection{Proof of Theorem 4.3 (Median)} 
Let $Y_k(j_0)=\sum_{j=1}^{j_0} X_k(j)$. Since we use the median to form the aggregate, we need to establish that 
$\sigma_{S_t}(t)=\min\{j: \frac{1}{m}\sum_{k\in[m]}Y_k(j)\geq 0.5\}$. According to Lemma~4.5 of the main text, we have $\prob{[Y_k(i)=1]}=1-\prob{[\sigma_{k,A}(x)>i]}\geq 1-\phi$ while $\prob{[Y_k(i-1)=1]}=\prob{[\sigma_{k,A}(x)<i]}\leq \phi$. Therefore, if $\phi<0.5$, using Hoeffding's inequality, we have 
\begin{align*}
\prob{[D_t]}\leq \prob{\left[\frac{1}{m}\sum_{k\in[m]}Y_k(i)< 0.5\right]} \quad+\prob{\left[\frac{1}{m}\sum_{k\in[m]}Y_k(i-1)> 0.5\right]} \leq 2e^{-2m(\frac{1}{2}-\phi)^2}.
\end{align*}
As a result, for $m\geq c\log\frac{2n}{\delta}$ with $c=\frac{2}{(1-2\phi)^2}$, we have $\prob{[\sigma=\sigma_0]}> 1-2ne^{-2m(\frac{1}{2}-\phi)^2}\geq 1-\delta$. 

\subsection{Proof of the Performance Guarantees for Partial Ranking Aggregation}
%In the partial ranking aggregation algorithm, there are two stages. The steps 1-3 in Algorithm 3 compute a complete order while the steps after 3 is to transform the complete order into a partial ranking. Therefore, to recover the centroid of the generalize mallows model with high probability, we will first prove the complete order in the first stage corresponds to some order after break ties of the centroid with high probability and then prove the transformation in the second stage will be exactly like the centroid with high probability. The proof of the first stage is like the proof for full ranking aggregation and the proof of the second stage can be shared by using both mode-based and median-based algorithms.

Denote the Lehmer code of the output permutation $\sigma$ of Algorithm 2 of the main text by $\hat{\bfc}_{\sigma}$. We say that the LCA algorithm succeeds if $\sigma$ is in $\Sigma_0$, which is equivalent to saying that $\hat{\bfc}_{\sigma}(t)\in[\bfc_{\sigma}(t),\bfc_{\sigma}'(t)]$. Given that $\hat{\bfc}_{\sigma}(1)=0=\bfc_{\sigma_0}(1)=\bfc_{\sigma}'(1)$, from the union bound, we have 
\begin{align*}
\prob{\left[\sigma\in\Sigma_0\right]}=\prob{\left[\hat{\bfc}_{\sigma}(t)\in[\bfc_{\sigma}(t),\bfc_{\sigma}'(t)], \forall t\right]}\geq 1-\sum_{t=2}^{n-1} \prob{\left[\hat{\bfc}_{\sigma}(t)\not\in [\bfc_{\sigma}(t),\bfc_{\sigma}'(t)]\right]}.
\end{align*}
In Section 4 of the main text, we explained how the Lehmer code transform $\hat{\bfc}_{\sigma}$ may be viewed as a form of InsertionSort, which in the $t$-th iteration places the element $t$ at the $(t-\hat{\bfc}_{\sigma}(t))$th position within the subset of elements $S_t=[t]$. With this choice of subset $S_t$, for any $\pi$, we have that $\pi_{S_t}(t)=t-\bfc_{\pi}(t)$. Hence, the event $\{\hat{\bfc}_{\sigma}(t)\not\in [\bfc_{\sigma}(t),\bfc_{\sigma}'(t)]\}$ is equivalent to the event $\{\sigma_{S_t}(t)<l_{\sigma_{0,S_t}(t)}\;\text{or}\;\sigma_{S_t}(t)>r_{\sigma_{0,S_t}(t)}\}$, which we denote by $D_t$. The proof reduces to finding a lower bound on $\prob{[D_t]}$.  

For convenience, we let $l\triangleq l_{\sigma_{0,S_t}(t)}$ and $r\triangleq r_{\sigma_{0,S_t}(t)}$. Given that the ranking $\sigma_k\in \Sigma$ is sampled from a GMM$(\sigma_0,\phi)$, we define the random variable $X_k(j)$ as the vote that $\sigma_k$ cast for $t$ to be at position $j$ in $S_t$. Then, $V(j)=\sum_{k\in[m]}X_k(j)$ is the total vote cast by all partial rankings in $\Sigma$ to rank $t$ at the $j$-th position.

\subsubsection{Proof of Theorem 4.6 (Median)} 
Let $Y_k(j_0)=\sum_{j=1}^{j_0} X_k(j)$. Since we use the median to form the aggregate, we have $\sigma_{S_t}(t)=\min\{j: \frac{1}{m}\sum_{k\in[m]}Y_k(j)\geq 0.5\}$. 
%In the following, we are to prove that $\prob{[l\geq\hat{c}(t)+1\geq r]}$ is close to 1, which is equivalent to $\prob{[D_t|D_{t-1}]}$ close to 1. The event $\{l \geq\hat{c}(t)+1\leq r\}$ is equivalent to $\{\frac{1}{m}\sum_{j=1}^r V(j)>0.5\}\cap\{\frac{1}{m}\sum_{j=l}^|A| V(j)>0.5\}$. Let us focus on the summation $\frac{1}{m}\sum_{j=1}^r V(j)=\frac{1}{m}\sum_{k\in[m]}\sum_{j=1}^rX_k(j)$ first.
Define the event $W=\{r_{\sigma_{k,S_t}(t)}\leq r\}$. When $W$ occurs, $\sigma_k$ contributes 1 to $Y_k(r)$. Let $Q=\cup_{j=1}^{n-r}Q_j$, where $Q_j=\{r_{\sigma_{k,S_t}(t)}=j+r,\;l_{\sigma_{k,S_t}(u)}\leq r\}$. When $Q_j$ occurs, $\sigma_k$ contributes a fractional vote $V_j$ to $Y_k(r)$, where $V_j=\frac{r-l_{\sigma_{k,S_t}(t)}+1}{r_{\sigma_{k,S_t}(t)}-l_{\sigma_{k,S_t}(t)}+1}\geq V_j'=\frac{1}{j+1}$. In fact, $V_j=V_j'$ when $l_{\sigma_{k,S_t}(t)}=r$. Therefore, based on the Lemma 4.7 of the main text, we have 
\begin{align}\label{apppartialmedian3}
\mathbb{E}{[Y_k(r)]}&\geq \prob{[W]} + \sum_{j=1}^{t-r}\frac{1}{j+1} \prob{[Q_j]}\\
&\geq 1-\frac{1}{2}\phi^{1/2}-\frac{1}{2}\phi.
\end{align}
Let $q'=1-\frac{1}{2}\phi^{1/2}-\frac{1}{2}\phi$. When $\phi^{1/2}+\phi<1$, it follows that $q'>0.5$. By using Hoeffding's inequality, we obtain
\begin{align*}
\prob{\left[\frac{1}{m}\sum_{k\in[m]}Y_k(r)<0.5\right]}\leq\exp(-2m(1/2-q')^2).
\end{align*}
Let $Z_k(j_0)=\sum_{j=j_0}^{t} X_k(j)$. In an analogous manner, we can prove that 
\begin{align*}
\prob{\left[\frac{1}{m}\sum_{k\in[m]}Z_k(l)<0.5\right]}\leq\exp(-2m(1/2-q')^2).
\end{align*}

Therefore, the probability of success of iteration $t$ may be bounded as
\begin{align*}
\prob{[D_t]}\leq& \prob{\left[\frac{1}{m}\sum_{k=1}^mY_k(r)<0.5\right]}+\prob{\left[\frac{1}{m}\sum_{k=1}^{m}Z_k(l)<0.5\right]}\leq 2e^{-2m(1/2-q')^2}.
\end{align*}
As a result, when $m\geq c\log\frac{2n}{\delta}$ with $c=\frac{2}{(1-2q')^2},$ where $q'=1-\frac{1}{2}\phi^{1/2}-\frac{1}{2}\phi$, we have $\prob{[\sigma\in\Sigma_0]}> 1-\delta$. 

%\subsubsection{Success of stage2}
%Given success of stage1 where the output order corresponds to a full ranking generated by breaking ties in $\sigma_0$, the second stage of partial ranking aggregation is to assign the relation between two adjacent elements of the output order of first stage.  

\section{Other Lemmas and Proofs}\label{oths}
\begin{lemma} \label{extLemma1}
Let $\sigma_0$ be a ranking over $S$ and let $A \subseteq S$ be such that $A$ contains the elements ranked highest in $\sigma_0$. Consider a ranking $\sigma\sim$MM$(\sigma_0,\phi)$. Then, the marginal distribution of $\sigma$ over $S/A$ is the distribution $MM(\sigma_{0,S/A},\phi)$.  
\end{lemma}
\begin{proof}
It suffices to prove the result for $A=\{{x\}}$, where $x$ is the element ranked highest in $\sigma_0$, as this result may be applied inductively. Consider all permutations $\sigma$ such that for $\sigma_{S/\{x\}}=\pi$ and some $j\in[|S|]$, one has
\begin{align*}
\sigma^{-1}(t)=\left\{\begin{array}{lc} \pi^{-1}(t), & 1\leq t<j,  \\ 
\pi^{-1}(t-1), & j<t\leq |S|, \\
x, & t=j.
\end{array}\right.
\end{align*}
For simplicity of notation, we use $\sigma^{(j)}$ to denote a permutation with the above property. Then,
\begin{align*}
\prob{[\sigma_{S/\{x\}}=\sigma']}=&\sum_{j=1}^{|S|}\prob{[\sigma^{(j)}]}=\frac{1}{Z_{|S|}}\sum_{j=1}^{|S|}\phi^{d_{\tau}(\sigma^{(j)},\sigma_0)}=\frac{1}{Z_{|S|}}\sum_{j=1}^{|S|}\phi^{j-1+d_{\tau}(\sigma_{S/\{x\}}^{(j)},\sigma_{0,S/\{x\}})}\\
=&\frac{\sum_{j=1}^{|S|}\phi^{j-1}}{Z_{|S|}}\phi^{d_{\tau}(\pi,\sigma_{0,S/\{x\}})}=\frac{1}{Z_{|S/\{x\}|}}\phi^{d_{\tau}(\pi,\sigma_{0,S/\{x\}})},
\end{align*}
where $Z_{n}=\prod_{i=1}^{n-1}\sum_{j=0}^i\phi^j$ denotes the normalization constant in the Mallows distribution of permutations with 
$n$ elements. 
\end{proof}
Observe that the same result holds when $A$ is assumed to contain the lowest ranked element in $\sigma_0$.

\section{Supplementary Algorithms}\label{supplementalg}
\subsection{Efficient Algorithms for Computing the Mode/Median for Partial Ranking Aggregation}

In Section 3 of the main text which discusses partial ranking aggregation, we pointed out that one can efficiently compute the 
voting function $V_x(y)$, and hence the mode/median $\hat{c}$ as well. Algorithm VII.1 of this text explains how to efficiently compute $V_x(y)$, provided that for fixed $k,x$, $v_{k\rightarrow x}(y)$ is positive over a continuous interval, or more precisely, when $[x-c_{\sigma_k}(x),\, x-c_{\sigma_k'}(x)]$. Algorithm VII.1 has complexity $\mathcal{O}(m+x)$ and the computation of the mode/median of the component $\hat{c}(x)$ requires $\mathcal{O}(x)$ time. Therefore, the total complexity of Algorithm 2 of the main text for partial rankings equals $\mathcal{O}(mn+n^2)$.

\begin{table}[htb]
\centering
\begin{tabular}{l}
\hline
\label{alg:LCaggregationPRvote}
\textbf{Algorithm VII.1: }\\ \textbf{Computing $\{V_{x}(y)\}_{y\in[x]}$ when $v_{k\rightarrow x}(y)$ is positive over $[x-c_{\sigma_k}(x),\, x-c_{\sigma_k'}(x)]$}\\
\textbf{Input:} $\bfc_{\sigma_k}, \bfc_{\sigma_k}'$, votes $v_{k\rightarrow x}(y)=v_{k\rightarrow x}1_{x-\bfc_{\sigma_k}'\leq y\leq x-\bfc_{\sigma_k}}$, $\forall\;k\in[m]$; \\
\ 1: \textbf{Initialize} $V_{x}(y)=0,$ for all $y\in[x+1]$; \\
\ 2: \textbf{For $k$ from $1$ to $m$ do}\\
\ 3: \quad $V_{x}(x-\bfc_{\sigma_k}'(x))=V_{x}(x-\bfc_{\sigma_k}'(x))+v_{k\rightarrow x}$; \\
\ 4: \quad $V_{x}(x+1-\bfc_{\sigma_k}(x))=V_{x}(x+1-\bfc_{\sigma_k}(x))-v_{k\rightarrow x}$; \\
\ 5: \textbf{For $k$ from $2$ to $x+1$ do}\\
\ 6: \quad $V_{x}(y)=V_{x}(y-1)+V_{x}(y)$; \\
\ 7: \textbf{Output:} $V_{x}(y)$;\\
\hline
\end{tabular}
\end{table}

\subsection{A Kemeny-Distance Optimal Algorithm for Transforming Permutations into Partial Rankings}
In Section 3 of the main text pertaining to partial ranking aggregation, we pointed out that one can optimally transform the permutation output of Algorithm 2 into a partial ranking. Algorithm VII.2 of this text explains how to perform this transform. 
In the description of the algorithm, we used $a:b =(a,a+1,...,b)$ where $a,b\in\mathbb{Z},\,b\geq a$. 
For a vector $V$, we used $V(a:b)$ to denote $(V(a),V(a+1),...,V(b))$. Algorithm VII.2 has complexity $\mathcal{O}(mn^2+n^3)$.

\begin{table}[htb]
\centering
\begin{tabular}{l}
\hline
\label{alg:LCaggregationPRvote}
\textbf{Algorithm VII.2: }\\ \textbf{Transforms a Permutation into a Partial Ranking that is Kemeny-Distance Optimal}\\
\textbf{Input:} Permutation $\sigma$, Set of partial rankings $\Sigma$; \\
\ 1: \textbf{Initialize} BucketSize$=(1,1,...,1)\in \mathbb{N}^n$;\\
\ 2: \textbf{Initialize} $W=\{w_{ij}\}_{i,j\in [n]}$ where $w_{ij}=\frac{1}{m}\sum_{k\in[m]}1_{\sigma_{k}(\sigma^{-1}(i))<\sigma_{k}(\sigma^{-1}(j))}$;\\
\ 3: [Val, BucketSize]=Dynamic-Programming($W$, BucketSize);\\
\ 4: Construct a partial ranking $\sigma'$ by putting the lowest BucketSize(1) many elements of $\sigma$ into $\mathcal{B}_{1}(\sigma')$; \\
\quad \quad Proceed by taking BucketSize(2) many elements of $\sigma$ and placing them into $\mathcal{B}_{2}(\sigma')$ and so on; \\
\ 5: \textbf{Output:} $\sigma'$.\\
\hline
\textbf{Dynamic-Programming($W$, BucketSize)} \\
\ 1: $n'=$length(BucketSize); \\
\ 2: \textbf{If $n'=1$} \\
\quad \quad \quad return [0, BucketSize]; \\
\ 3: $s=\lfloor n'/2\rfloor;$\\
When $\sigma^{-1}(s)$ and $\sigma^{-1}(s+1)$ are in different buckets (4-6)  \\
\ 4: [Val1, BucketSize1]=Dynamic-Programming($W(1:s,1:s)$, BucketSize($1:s$)); \\
\ 5: [Val2, BucketSize2]=Dynamic-Programming($W(s+1:n',s+1:n')$, BucketSize($s+1:n$)); \\
\ 6: Val-div=Val1+Val2+$\sum_{i=s+1}^{n'}\sum_{j=1}^s w_{ij}+\frac{1}{2}\sum_{i=s+1}^{n'}\sum_{j=1}^s (\text{BucketSize}(i)*\text{BucketSize}(j)-w_{ji}-w_{ij})$.\\
When $\sigma^{-1}(s)$ and $\sigma^{-1}(s+1)$ are in the same bucket (7-13)  \\
\ 7: $w_{si}=w_{si}+w_{(s+1)i}$, for all $i\in[n']$;\\
\ 8: $w_{is}=w_{is}+w_{i(s+1)}$, for all $i\in[n']$; \\
\ 9: Val3=$1/2*w_{ss}$ ; \\
\ 10: Construct $W'\in \mathbb{R}^{n'-1\times n'-1}$ by deleting the $s+1$th row and $s+1$th column of $W$; \\
\ 11: Construct  newBucketSize: \\
\quad \quad \quad \quad \quad \quad \quad \quad \quad \quad \quad  newBucketSize($i$)=BucketSize($i$) for $1\leq i\leq s$; \\
\quad \quad \quad \quad \quad \quad \quad \quad \quad \quad \quad newBucketSize($i$)=BucketSize($i+1$) for $s+1\leq i\leq n'-1$; \\
\quad \quad \quad \quad  \quad \quad \quad \quad \quad \quad \quad newBucketSize($s$)=BucketSize($s$)+BucketSize($s+1$); \\
\ 12:  [Val4, BucketSize3]=Dynamic-Programming($W'$, newBucketSize);\\
\ 13: Val-con= Val3+Val4; \\
\ 14: \textbf{if} Val-con$>$Val-div, \\
\quad \quad \quad Construct BucketSize4 via concatenation of BucketSize1 and BucketSize2; \\
\quad \quad \quad return [Val-div, BucketSize4]; \\
\ 15: \textbf{else} return [Val-con, BucketSize3]; \\
\hline
\end{tabular}
\end{table}

%\bibliographystyle{IEEE.bst}
%\bibliography{Supplement.bib}

{\normalsize\bf Reference}
\begin{thebibliography}{99}


\bibitem{burges2005learning}
Chris Burges, Tal Shaked, Erin Renshaw, Ari Lazier, Matt Deeds, Nicole
  Hamilton, and Greg Hullender,
\newblock ``Learning to rank using gradient descent,''
\newblock in {\em Proceedings of the 22nd international conference on Machine
  learning}. ACM, 2005, pp. 89--96.

\bibitem{liu2009learning}
Tie-Yan Liu,
\newblock ``Learning to rank for information retrieval,''
\newblock {\em Foundations and Trends in Information Retrieval}, vol. 3, no. 3,
  pp. 225--331, 2009.

\bibitem{kim2014hydra}
Minji Kim, Farzad Farnoud, and Olgica Milenkovic,
\newblock ``Hydra: gene prioritization via hybrid distance-score rank
  aggregation,''
\newblock {\em Bioinformatics}, p. btu766, 2014.

\bibitem{negahban2012iterative}
Sahand Negahban, Sewoong Oh, and Devavrat Shah,
\newblock ``Iterative ranking from pair-wise comparisons,''
\newblock in {\em Advances in Neural Information Processing Systems}, 2012, pp.
  2474--2482.

\bibitem{chen2013pairwise}
Xi~Chen, Paul~N Bennett, Kevyn Collins-Thompson, and Eric Horvitz,
\newblock ``Pairwise ranking aggregation in a crowdsourced setting,''
\newblock in {\em Proceedings of the sixth ACM international conference on Web
  search and data mining}. ACM, 2013, pp. 193--202.

\bibitem{kemeny1959mathematics}
John~G Kemeny,
\newblock ``Mathematics without numbers,''
\newblock {\em Daedalus}, vol. 88, no. 4, pp. 577--591, 1959.

\bibitem{davenport2004computational}
Andrew Davenport and Jayant Kalagnanam,
\newblock ``A computational study of the kemeny rule for preference
  aggregation,''
\newblock in {\em AAAI}, 2004, vol.~4, pp. 697--702.

\bibitem{dwork2001rank}
Cynthia Dwork, Ravi Kumar, Moni Naor, and D~Sivakumar,
\newblock ``Rank aggregation revisited,'' 2001.

\bibitem{ailon2008aggregating}
Nir Ailon, Moses Charikar, and Alantha Newman,
\newblock ``Aggregating inconsistent information: ranking and clustering,''
\newblock {\em Journal of the ACM (JACM)}, vol. 55, no. 5, pp. 23, 2008.

\bibitem{diaconis1977spearman}
Persi Diaconis and Ronald~L Graham,
\newblock ``Spearman's footrule as a measure of disarray,''
\newblock {\em Journal of the Royal Statistical Society. Series B
  (Methodological)}, pp. 262--268, 1977.

\bibitem{kenyon2007rank}
Claire Kenyon-Mathieu and Warren Schudy,
\newblock ``How to rank with few errors,''
\newblock in {\em Proceedings of the thirty-ninth annual ACM symposium on
  Theory of computing}. ACM, 2007, pp. 95--103.

\bibitem{fligner1993probability}
Michael~A Fligner and Joseph~S Verducci,
\newblock {\em Probability models and statistical analyses for ranking data},
  vol.~80,
\newblock Springer, 1993.

\bibitem{caron2012efficient}
Francois Caron and Arnaud Doucet,
\newblock ``Efficient bayesian inference for generalized bradley--terry
  models,''
\newblock {\em Journal of Computational and Graphical Statistics}, vol. 21, no.
  1, pp. 174--196, 2012.

\bibitem{lu2011learning}
Tyler Lu and Craig Boutilier,
\newblock ``Learning mallows models with pairwise preferences,''
\newblock in {\em Proceedings of the 28th International Conference on Machine
  Learning (ICML-11)}, 2011, pp. 145--152.

\bibitem{lebanon2002cranking}
Guy Lebanon and John Lafferty,
\newblock ``Cranking: Combining rankings using conditional probability models
  on permutations,''
\newblock in {\em ICML}. Citeseer, 2002, vol.~2, pp. 363--370.

\bibitem{fagin2004comparing}
Ronald Fagin, Ravi Kumar, Mohammad Mahdian, D~Sivakumar, and Erik Vee,
\newblock ``Comparing and aggregating rankings with ties,''
\newblock in {\em Proceedings of the twenty-third ACM SIGMOD-SIGACT-SIGART
  symposium on Principles of database systems}. ACM, 2004, pp. 47--58.

\bibitem{stanley2011enumerative}
Richard~P Stanley,
\newblock {\em Enumerative combinatorics},
\newblock Number~49. Cambridge university press, 2011.

\bibitem{marevs2007linear}
Martin Mare{\v{s}} and Milan Straka,
\newblock ``Linear-time ranking of permutations,''
\newblock in {\em Algorithms--ESA 2007}, pp. 187--193. Springer, 2007.

\bibitem{myrvold2001ranking}
Wendy Myrvold and Frank Ruskey,
\newblock ``Ranking and unranking permutations in linear time,''
\newblock {\em Information Processing Letters}, vol. 79, no. 6, pp. 281--284,
  2001.

\bibitem{vajnovszki2013lehmer}
Vincent Vajnovszki,
\newblock ``Lehmer code transforms and mahonian statistics on permutations,''
\newblock {\em Discrete Mathematics}, vol. 313, no. 5, pp. 581--589, 2013.

\bibitem{coppersmith2006ordering}
Don Coppersmith, Lisa Fleischer, and Atri Rudra,
\newblock ``Ordering by weighted number of wins gives a good ranking for
  weighted tournaments,''
\newblock in {\em Proceedings of the seventeenth annual ACM-SIAM symposium on
  Discrete algorithm}. Society for Industrial and Applied Mathematics, 2006,
  pp. 776--782.

\bibitem{dwork2001rankw}
Cynthia Dwork, Ravi Kumar, Moni Naor, and Dandapani Sivakumar,
\newblock ``Rank aggregation methods for the web,''
\newblock in {\em Proceedings of the 10th international conference on World
  Wide Web}. ACM, 2001, pp. 613--622.

\bibitem{kambatla2012efficient}
Karthik Kambatla, Georgios Kollias, and Ananth Grama,
\newblock ``Efficient large-scale graph analysis in mapreduce,'' 2012.

\bibitem{kamishima2003nantonac}
Toshihiro Kamishima,
\newblock ``Nantonac collaborative filtering: recommendation based on order
  responses,''
\newblock in {\em Proceedings of the ninth ACM SIGKDD international conference
  on Knowledge discovery and data mining}. ACM, 2003, pp. 583--588.

\bibitem{goldberg2001eigentaste}
Ken Goldberg, Theresa Roeder, Dhruv Gupta, and Chris Perkins,
\newblock ``Eigentaste: A constant time collaborative filtering algorithm,''
\newblock {\em Information Retrieval}, vol. 4, no. 2, pp. 133--151, 2001.

\bibitem{harper2016movielens}
F~Maxwell Harper and Joseph~A Konstan,
\newblock ``The movielens datasets: History and context,''
\newblock {\em ACM Transactions on Interactive Intelligent Systems (TiiS)},
  vol. 5, no. 4, pp. 19, 2016.

\bibitem{awasthi2014learning}
Pranjal Awasthi, Avrim Blum, Or~Sheffet, and Aravindan Vijayaraghavan,
\newblock ``Learning mixtures of ranking models,''
\newblock in {\em Advances in Neural Information Processing Systems}, 2014, pp.
  2609--2617.
  
\end{thebibliography}

\end{document}